\documentclass[11pt]{article}
\usepackage{amsmath,amsbsy,amsfonts,amssymb,amsthm}
\usepackage{dsfont,fullpage,units}
\usepackage{graphicx}
\usepackage{algorithm,algorithmic,mathtools}
\usepackage{color,cases}
\usepackage{tikz}
\usetikzlibrary{calc,shapes}
\usepackage{subfigure}
\usepackage{hyperref}
\usepackage{float}
\usepackage{authblk}

\newtheorem{claim}{Claim}

\newtheorem{theorem}{Theorem}[section]
\newtheorem{proposition}[theorem]{Proposition}
\newtheorem{corollary}[theorem]{Corollary}

\newtheorem{lemma}{Lemma}

\theoremstyle{definition}

\newtheorem{definition}{Definition}

\newtheorem{remark}{Remark}

\newcommand\inner[1]{\langle #1 \rangle}

\newcommand\poly{\operatorname{poly}}

\DeclareMathOperator{\erf}{erf}
\newcommand{\Exp}{\mathop{\mathbb E}\displaylimits}
\newcommand{\E}{\mathbb E}

\newcommand{\eps}{\epsilon}

\def\shownotes{0}  
\ifnum\shownotes=1
\newcommand{\authnote}[2]{{ $\ll$\textsf{\footnotesize #1 notes: #2}$\gg$}}
\else
\newcommand{\authnote}[2]{}
\fi
\newcommand{\Tnote}[1]{{\authnote{\color{blue}Tengyu}{#1}}}

\newcommand{\I}{\textup{I}}
\newcommand{\h}{\textup{h}}
\newcommand{\KL}{\textup{D}_{\textup{kl}}}
\newcommand{\TV}[1]{\|#1\|_{\textrm{TV}}}
\newcommand{\ent}{\textup{Ent}}

\newcommand{\calD}{\mu}
\newcommand{\calP}{\mathcal{P}}
\newcommand{\bs}[1]{\boldsymbol{#1}}

\newcommand{\Bern}{B}

\newcommand{\ic}{\textrm{IC}}
\newcommand{\mic}{\textup{min-IC}}
\newcommand{\cc}{\textrm{CC}}

\floatstyle{boxed}
\newfloat{Protocol}{ht}{pro}

\newcommand{\ignore}[1]{}
\newcommand{\Task}{T}
\newcommand{\Taskdet}{T_{det}}
\newcommand{\SGME}{\textup{SGME}}
\newcommand{\GD}{\textup{GD}}
\newcommand{\GME}{\textup{GME}}
\newcommand{\SLR}{\textup{SLR}}
\newcommand{\wtX}{\widetilde{X}}

\newcommand{\pub}{R_{\textup{pub}}}
\title{Communication Lower Bounds for Statistical Estimation Problems via a Distributed Data Processing Inequality}
\author[1]{Mark Braverman}
\author[1]{Ankit Garg}
\author[1]{Tengyu Ma}
\author[2]{Huy L. Nguyen}
\author[3]{David P. Woodruff}
\affil[1]{Princeton University}
\affil[2]{Toyota Technological Institute at Chicago}
\affil[3]{IBM Research Almaden}
\begin{document}
\maketitle
	\setcounter{page}{0}
\abstract{

	We study the tradeoff between the statistical error and communication cost of distributed statistical estimation problems in high dimensions. In the distributed sparse Gaussian mean estimation problem, each of the $m$ machines receives $n$ data points from a $d$-dimensional Gaussian distribution with unknown mean $\theta$ which is promised to be $k$-sparse.  The machines communicate by message passing and aim to estimate the mean $\theta$. We provide a tight (up to logarithmic factors) tradeoff between the estimation error and the number of bits communicated between the machines. This directly leads to a lower bound for the distributed \textit{sparse linear regression} problem: to achieve the statistical minimax error, the total communication is at least $\Omega(\min\{n,d\}m)$, where $n$ is the number of observations that each machine receives and $d$ is the ambient dimension. These lower results improve upon~\cite{Shamir14online,duchi15} by allowing multi-round iterative communication model. 
	We also give the first optimal simultaneous protocol in the dense case for mean estimation. 
	
	As our main technique, we prove a \textit{distributed data processing inequality}, as a generalization of usual data processing inequalities, which might be of independent interest and useful for other problems.

	
	
	}
	\newpage

\section{Introduction}

Rapid growth in the size of modern data sets has fueled a lot of interest in solving statistical and machine learning tasks in a distributed environment using multiple machines. 
Communication between the machines has emerged as an important resource and sometimes the main bottleneck.  A lot of recent work has been devoted to design communication-efficient learning algorithms~\cite{Duchi12, DBLP:journals/jmlr/ZhangDW13, ZhangXiao, kvw14, lbkw14, SSZ14, LSQT15}.  

In this paper we consider statistical estimation problems in the distributed setting, which can be formalized as follows. There is a family of distributions $\mathcal{P} = \{\mu_{\theta}: \theta\in \Omega \subset\mathbb{R}^d\}$ that is parameterized by $\theta\in \mathbb{R}^d$. Each of the $m$ machines is given $n$ i.i.d samples drawn from an unknown distribution $\mu_{\theta}\in \mathcal{P}$. The machines communicate with each other by message passing, and do computation on their local samples and the messages that they receives from others.  Finally one of the machines needs to output an estimator $\hat{\theta}$ and the statistical error is usually measured by the mean-squared loss $\Exp[\|\hat{\theta} -\theta\|^2]$. We count the communication between the machines in bits. 

This paper focuses on understanding the fundamental tradeoff between communication and the statistical error for high-dimensional statistical estimation problems. Modern large datasets are often equipped with a high-dimensional statistical model, while communication of high dimensional vectors could potentially be expensive.  It has been shown by Duchi et al.~\cite{DBLP:journals/corr/DuchiJWZ14} and Garg et al.~\cite{DBLP:conf/nips/GargMN14} that for the linear regression problem, the communication cost must scale with the dimensionality for achieving optimal statistical minimax error -- not surprisingly, the machines have to communicate high-dimensional vectors in order to estimate high-dimensional parameters. 

These negative results naturally lead to the interest in high-dimensional estimation problems with additional sparse structure on the parameter $\theta$. It has been well understood that the statistical minimax error typically depends on the intrinsic dimension, that is, the sparsity of the parameters, instead of the ambient dimension\footnote{the dependency on the ambient dimension is typically logarithmic.}. Thus it is natural to expect that the same phenomenon also happens for communication. 

However, this paper disproves this possibility in the interactive communication model by proving that for the \textit{sparse Gaussian mean estimation} problem (where one estimates the mean of a Gaussian distribution which is promised to be sparse, see Section~\ref{sec:setup} for the formal definition), in order to achieve the statistical minimax error, the communication must scale with the ambient dimension. 
On the other end of the spectrum, if alternatively the communication only scales with the sparsity, then the statistical error must scale with the ambient dimension (see Theorem~\ref{thm:sparse-gaussian-mean}). Shamir~\cite{Shamir14online} establishes the same result for the 1-sparse case under a non-iterative communication model. 

Our lower bounds for the Gaussian mean estimation problem imply lower bounds for the \textit{sparse linear regression} problem (Corollary~\ref{cor:sparse-linear-regression}) via the reduction of~\cite{ZDJW13}: for a Gaussian design matrix, to achieve the statistical minimax error, the communication cost per machine needs to be $\Omega(\min\{n,d\})$ where $d$ is the ambient dimension and $n$ is the dimension of the observation that each machine receives. 
This lower bound matches the upper bound in~\cite{LSQT15} when $n$ is larger than $d$. When $n$ is less than $d$, we note that it is not clear whether $O(n)$ or $O(d)$ should be the minimum communication cost per machine needed. In any case, our contribution here is in proving a lower bound that does not depend on the sparsity. Compared to previous work of Steinhardt and Duchi~\cite{duchi15}, which proves the same lower bounds for a memory-bounded model, our results work for a stronger communication model where multi-round iterative communication is allowed. Moreover, our techniques are possibly simpler and potentially easier to adapt to related problems. For example, we show that the result of Woodruff and Zhang \cite{WZ12} on the information complexity of distributed gap majority can be reproduced by our technique with a cleaner proof (see Theorem~\ref{thm:gap_majority}).  

We complement our lower bounds for this problem in the dense case by providing a new simultaneous protocol, improving the number of rounds of the previous communication-optimal protocol from $O(\log m)$ to $1$ (see Theorem~\ref{thm:one-way-upper-bound}). Our protocol is based on a certain combination of many bits from 
a few Gaussian samples, together with roundings (to a single bit) of the fractional parts of many 
Gaussian samples. 
 

Our proof techniques are potentially useful for other questions along these lines. We first use a modification of the direct-sum result of~\cite{DBLP:conf/nips/GargMN14}, which is tailored towards sparse problems, to reduce the estimation problem to a detection problem. Then we prove what we call a \textit{distributed data processing inequality} for bounding from below the cost of the detection problem. The latter is the crux of our proofs.
We elaborate more on it in the next subsection. 

\subsection{Distributed Data Processing Inequality}

We consider the following distributed detection problem. As we will show in Section~\ref{sec:gaussian-mean} (by a direct-sum theorem), it suffices to prove a tight lower bound in this setting, in order to prove a lower bound on the communication cost for the sparse linear regression problem. 

\noindent{\bf Distributed detection problem: } We have a family of distributions $\mathcal{P}$ that consist of only two distributions $\{\mu_0,\mu_1\}$, and the parameter space $\Omega = \{0,1\}$. To facilitate the use of tools from information theory, sometimes it is useful to introduce a prior over the parameter space. Let $V\sim \Bern_q$ be a Bernoulli random variable with probability $q$ of being $1$. Given $V =v \in \{0,1\}$, we draw i.i.d. samples $X_1,\dots,X_m$ from $\mu_v$ and the $j$-th machine receives one sample $X_j$, for $j =1,\dots,m$. We use $\Pi\in \{0,1\}^*$ to denote the sequences of messages that are communicated by the machines.  We will refer to $\Pi$ as a ``transcript", and the distributed algorithm that the machines execute as a ``protocol". 

The final goal of the machines is to output an estimator for the hidden parameter $v$ which is as accurate as possible. 
We formalize the estimator as a (random) function $\hat{v}: \{0,1\}^* \rightarrow \{0,1\}$ that takes the transcript $\Pi$ as input. We require that given $V = v$, the estimator is correct with probability at least $3/4$, that is, $\min_{v\in \{0,1\}} \Pr[\hat{v}(\Pi) = v\mid V=v] \ge 3/4$. 
When $q = 1/2$, this is essentially equivalent to the statement that the transcript $\Pi$ carries $\Omega(1)$ information about the random variable $V$. Therefore, the mutual information $\I(V;\Pi)$ is also used as a convenient measure for the quality of the protocol when $q = 1/2$.  

\noindent{\bf Strong data processing inequality: } The mutual information viewpoint of the accuracy naturally leads us to the following approach for studying the simple case when $m = 1$ and $q = 1/2$. When $m=1$, we note that the parameter $V$, data $X$, and transcript $\Pi$ form a simple Markov chain $V\rightarrow X \rightarrow \Pi$. The channel $V\rightarrow X$ is defined as $X\sim \mu_v$, conditioned on $V = v$. 
The strong data processing inequality (SDPI) captures the relative ratio between $\I(V;\Pi)$ and $\I(X;\Pi)$. 

\begin{definition}[Special case of SDPI]\label{def:sdpi-intro} Let $V\sim B_{1/2}$ and the channel $V\rightarrow X$ be defined as above. Then there exists a constant $\beta\le 1$ that depends on $\mu_0$ and $\mu_1$, such that for any $\Pi$ that depends only on $X$ (that is, $V\rightarrow X\rightarrow \Pi$ forms a Markov Chain), we have 
	\begin{equation}
	\I(V;\Pi)\le \beta \cdot \I(X;\Pi) \label{eqn:data-processing-inequality-intro}. 
	\end{equation}
	
	An inequality of this type is typically referred to as a \textit{strong data processing inequality for mutual information} when $\beta < 1$ \footnote{Inequality~\eqref{eqn:data-processing-inequality-intro} is always true for a Markov chain $V\rightarrow X \rightarrow \Pi$ with $\beta  = 1$ and this is called the data processing inequality.}. Let $\beta(\mu_0,\mu_1)$ be the infimum over all possible $\beta$ such that ~\eqref{eqn:data-processing-inequality-intro} is true, which we refer to as the {\bf SDPI constant}. 
\end{definition}
%

 Observe that the LHS of~\eqref{eqn:data-processing-inequality-intro} measures how much information $\Pi$ carries about $V$, which is closely related to the accuracy of the protocol. The RHS of~\eqref{eqn:data-processing-inequality-intro} is a lower bound on the expected length of $\Pi$, that is, the expected communication cost. Therefore the inequality relates two quantities that we are interested in - the statistical quality of the protocol and the communication cost of the protocol. Concretely, when $q = 1/2$, in order to recover $V$ from $\Pi$, we need that $\I(V;\Pi) \ge \Omega(1)$, and therefore inequality~\eqref{eqn:data-processing-inequality-intro} gives that $\I(X;\Pi)\ge \Omega(\beta^{-1})$. Then it follows from Shannon's source coding theory that the expected length of $\Pi$ (denoted by $|\Pi|$) is bounded from below by $\Exp[|\Pi|] \ge \Omega(\beta^{-1})$. We refer to~\cite{Raginsky14} for a thorough survey of SDPI.\footnote{Also note that in information theory, SDPI is typically interpreted as characterizing how information decays when passed through the reverse channel $X\rightarrow V$. That is, when the channel $X\rightarrow V$ is lossy, then information about $\Pi$ will decay by a factor of $\beta$ after passing $X$ through the channel. However, in this paper we take a different interpretation that is more convenient for our applications.} 

In the multiple machine setting, Duchi et al.~\cite{DBLP:journals/corr/DuchiJWZ14} links the distributed detection problem with SDPI by showing from scratch that for any $m$, when $q = 1/2$, if $\beta$ is such that $(1-\sqrt{\beta})\mu_1\le \mu_0 \le (1+\sqrt{\beta})\mu_1$, then 
\begin{equation*}
\I(V;\Pi)\le \beta \cdot \I(X_1\dots X_m;\Pi). 
\end{equation*}
This results in the bounds for the Gaussian mean estimation problem and the linear regression problem. The main limitation of this inequality is that it requires the prior $B_q$ to be unbiased (or close to unbiased). For our target application of high-dimensional problems with sparsity structures, like sparse linear regression, in order to apply this inequality we need to put a very biased  prior $B_q$ on $V$. The proof technique of~\cite{DBLP:journals/corr/DuchiJWZ14} seems also hard to extend to this case with a tight bound\footnote{We note, though, that it seems possible to extend the proof to the situation where there is only one-round of communication.}. Moreover, the relation between $\beta$, $\mu_0$ and $\mu_1$ may not be necessary (or optimal), and indeed for the Gaussian mean estimation problem, the inequality is only tight up to a logarithmic factor, while potentially in other situations the gap is even larger. 

Our approach is essentially a prior-free multi-machine SDPI, which has the same SDPI constant $\beta$ as is required for the single machine one. We prove that, as long as the SDPI~\eqref{eqn:data-processing-inequality-intro} for a single machine is true with parameter $\beta$, and $\mu_0 \le O(1)\mu_1$, then the following prior-free multi-machine SDPI is true with the same constant $\beta$ (up to a constant factor). 

\begin{theorem}[Distributed SDPI]\label{thm:main-intro}
	Suppose $\frac{1}{c}\cdot \mu_0\le \mu_1\le c \mu_0$ for some constant $c\ge 1$, and let $\beta(\mu_0,\mu_1)$ be the SDPI constant defined in Definition~\ref{def:sdpi-intro}. 
	 Then in the distributed detection problem, 
	we have the following distributed strong data processing inequality, 
	\begin{equation}
		\h^2(\Pi\vert_{V=0}, \Pi\vert_{V=1}) \le  Kc\beta(\mu_0,\mu_1) \cdot \min\{\I(X_1\dots X_m;\Pi\mid V=0), \I(X_1\dots X_m;\Pi\mid V=1)\}\label{eqn:dist-sdpi-intro}
	\end{equation}
	where $K$ is a universal constant, and  $\h(\cdot,\cdot)$ is the Hellinger distance between two distributions and $\Pi\vert_{V=v}$ denotes the distribution of $\Pi$ conditioned on $V=v$. 
	
	Moreover, 
	for any $\mu_0$ and $\mu_1$ which satisfy the condition of the theorem, there exists a protocol that produces transcript $\Pi$ such that~\eqref{eqn:dist-sdpi-intro} is tight up to a constant factor. 
\end{theorem}
As an immediate consequence, we obtain a lower bound on the communication cost for the distributed detection problem. 
\begin{corollary}\label{cor:main}
	Suppose the protocol and estimator $(\Pi,\hat{v})$ are such that for any $v\in \{0,1\}$, given $V = v$ , the estimator $\hat{v}$ (that takes $\Pi$ as input) can recover $v$ with probability $3/4$. Then $$\max_{v\in \{0,1\}}\Exp[|\Pi|\mid V=v]\ge \Omega(\beta^{-1}).$$
\end{corollary}

Our theorem suggests that to bound the communication cost of the multi-machine setting from below, one could simply work in the single machine setting and obtain the right SDPI constant $\beta$. Then, a lower bound of $\Omega(\beta^{-1})$ for the multi-machine setting immediately follows. In other words, multi-machines need to communicate a lot to fully exploit the $m$ data points they receive ($1$ on each single machine) regardless of however complicated their multi-round protocol is.

\begin{remark}
	Note that our inequality differs from the typical data processing inequality on both the left and right hand sides. First of all, the RHS of~\eqref{eqn:dist-sdpi-intro} is always less than or equal to $\I(X_1\dots X_m; \Pi\mid V)$ for any prior $B_q$ on $V$. This allows us to have a tight bound on the expected communication $\Exp[|\Pi|]$ for the case when $q$ is very small. 
	
	Second, the squared Hellinger distance (see Definition~\ref{def:hellinger}) on the LHS of~\eqref{eqn:dist-sdpi-intro} is not very far away from $\I(\Pi;V)$, especially for the situation that we consider. It can be viewed as an alternative (if not more convenient) measure of the quality of the protocol than mutual information -- the further $\Pi\vert_{V=0}$ from $\Pi\vert_{V=1}$, the easier it is to infer $V$ from $\Pi$. 
    When a good estimator is possible (which is the case that we are going to apply the bound in), Hellinger distance, total variation distance between $\Pi\vert_{V=0}$ and $\Pi\vert_{V=1}$, and $\I(V;\Pi)$ are all $\Omega(1)$. Therefore in this case, the Hellinger distance does not make the bound weaker.  
	
	Finally, suppose we impose a uniform prior for $V$. Then the squared Hellinger distance is within a constant factor of $\I(V;\Pi)$ (see Lemma~\ref{lem:info-hellinger}, and the lower bound side was proved by~\cite{BJKS04}), 
	$$2\h^2(\Pi\vert_{V=0}, \Pi\vert_{V=1})  \ge \I(V;\Pi)\ge \h^2(\Pi\vert_{V=0}, \Pi\vert_{V=1})\,.$$
	Therefore, in the unbiased case,~\eqref{eqn:dist-sdpi-intro} implies the typical form of the data processing inequality. 
	
	
\end{remark}


\begin{remark}
	The tightness of our inequality does not imply that there is a protocol that solves the distributed detection problem with communication cost (or information cost) $O(\beta^{-1})$. We only show that inequality~\eqref{eqn:dist-sdpi-intro} is tight for some protocol but solving the problem requires having a protocol such that~\eqref{eqn:dist-sdpi-intro} is tight and that $\h^2(\Pi\vert_{V=0}, \Pi\vert_{V=1}) = \Omega(1)$. In fact, a protocol for which inequality~\eqref{eqn:dist-sdpi-intro} is tight is one in which only a single machine sends a message $\Pi$ which maximizes $\I(\Pi;V)/\I(\Pi;X)$. 
\end{remark}

\noindent{\bf Organization of the paper: } Section~\ref{sec:setup} formally sets up our model and problems and introduces some preliminaries. Then we prove our main theorem in Section~\ref{sec:main-theorem}.  In Section~\ref{sec:gaussian-mean} we state the main applications of our theory to the sparse Gaussian mean estimation problem and to the sparse linear regression problem. The next three sections are devoted to the proofs of results in Section~\ref{sec:gaussian-mean}.   In Section~\ref{sec:direct-sum}, we prove Theorem~\ref{thm:direct_sum_sparse} and in Section~\ref{sec:app:thm:multi-dpi-gaussian} we prove Theorem~\ref{thm:multi-dpi-gaussian} and Corollary~\ref{cor:sparse-linear-regression}.  
 In Section~\ref{sec:data-processing-ineq} we provide tools for proving single machine strong data processing inequality and prove Theorem~\ref{thm:truncated-guassian-dpi-info}. In Section~\ref{sec:one-way} we
present our matching upper bound in the simultaneous communication model. In section~\ref{sec:app:gap_majority} we give a simple proof of distributed gap majority problems using our machinery. 

\section{Problem Setup, Notations and Preliminaries}\label{sec:setup}
\subsection{Distributed Protocols and Parameter Estimation Problems}

Let $\calP = \{\mu_{\theta}: \theta\in \Omega\}$ be a family of distributions over some space $\mathcal{X}$, and $\Omega\subset\mathbb{R}^d$ be the space of all possible parameters. There is an unknown distribution $\mu_{\theta}\in \mathcal{P}$, and our goal is to estimate a parameter $\theta$ using $m$ machines. Machine $j$ receives $n$ i.i.d samples $X^{(1)}_j,\dots, X^{(n)}_j$ from distribution $\mu_{\theta}$. For simplicity we will use $X_j$ as a shorthand for all the samples machine $j$ receives, that is, $X_j = (X^{(1)}_j,\dots, X^{(n)}_j)$. Therefore $X_j\sim \mu_{\theta}^n$, where $\mu^n$ denotes the product of $n$ copies of $\mu$.  When it is clear from context, we will use $X$ as a shorthand for $(X_1,\dots,X_m)$.  We define the problem of estimating parameter $\theta$ in this distributed setting formally as task $\Task(n,m,\mathcal{P})$. When $\Omega = \{0,1\}$, we call this a detection problem and refer it to as $\Taskdet(n,m,\mathcal{P})$. 


The machines communicate via a publicly shown blackboard. That is, when a machine writes a message on the blackboard, all other machines can see the content. The messages that are written on the blackboard are counted as communication between the machines. Note that this model captures both point-to-point communication as well as broadcast communication.
Therefore, our lower bounds in this model apply to both the message passing setting and the broadcast setting.

We denote the collection of all the messages written on the blackboard by $\Pi$. We will refer to $\Pi$ as the
transcript and note that $\Pi\in \{0, 1\}^*$
is written in bits and the communication cost is defined as the
length of $\Pi$, denoted by $|\Pi|$. We will call the algorithm that the machines follow to produce $\Pi$ a protocol. With a slight abuse of notation, we ue $\Pi$ to denote both the protocol and the transcript produced by the protocol. 

One of the machines needs to estimate the value of $\theta$ using an estimator $\hat{\theta} :\{0,1\}^*\rightarrow \mathbb{R}^d$ which takes $\Pi$ as input. The accuracy of the estimator on $\theta$ is measured by the mean-squared loss: 
$$R((\Pi,\hat{\theta}), \theta) = \Exp\left[ \|\hat{\theta}(\Pi) - \theta\|^2_2\right],$$
where the expectation is taken over the randomness of the data $X$, and the estimator $\hat{\theta}$. The error of the estimator is the supremum of the loss over all $\theta$, 
\begin{equation}
R(\Pi,\hat{\theta}) = \sup_{\theta\in \Omega}\Exp\left[ \|\hat{\theta}(\Pi) - \theta\|^2_2\right].\label{eqn:def-loss}
\end{equation}
The communication cost of a protocol is measured by the expected length of the transcript $\Pi$, that is, 
$
\cc(\Pi) = \sup_{\theta\in \Omega} \Exp[|\Pi|]\label{eqn:def-cc}
$. 
The information cost $\ic$ of a protocol is defined as the mutual information between transcript $\Pi$ and the data $X$, 
\begin{equation}
\ic(\Pi) = \sup_{\theta\in \Omega} \I_{\theta}(\Pi;X\mid \pub) \label{eqn:def-ic}
\end{equation}
where $\pub$ denotes the public coin used by the algorithm and $I_{\theta}(\Pi;X\mid \pub)$ denotes the mutual information between random variable $X$ and $\Pi$ when the data $X$ is drawn from distribution $\mu_{\theta}$. We will drop the subscript $\theta$ when it is clear from context. 

For the detection problem, we need to define minimum information cost, a stronger version of information cost
\begin{equation}
\mic(\Pi) = \min_{v\in \{0,1\}} \I_{v}(\Pi;X\mid \pub) \label{eqn:def-min-ic}
\end{equation}
 
 \begin{definition}
 	We say that a protocol and estimator pair $(\Pi, \hat{\theta})$ solves the distributed estimation problem $T(m,n,d,\Omega,\mathcal{P})$ with information cost $I$, communication cost $C$, and mean-squared loss $R$ if $\ic(\Pi)\le I$, $\cc(\Pi)\le C$ and $R(\Pi,\hat{\theta}) \le R$. 
 \end{definition}
 When $\Omega = \{0,1\}$, we have a detection problem, and we typically use $v$ to denote the parameter and $\hat{v}$ as the (discrete) estimator for it. We define the communication and information cost the same as~\eqref{eqn:def-cc} and~\eqref{eqn:def-ic}, while defining the error in a more meaningful and convenient way,  
 \begin{equation*}
 R_{det}(\Pi, \hat{v}) = \max_{v\in \{0,1\}} \Pr[\hat{v}(\Pi) \neq v \mid V = v]
 \end{equation*}
\begin{definition}
 	We say that a protocol and estimator pair $(\Pi, \hat{v})$ solves the distributed detection problem $\Taskdet(m,n,d,\Omega,\mathcal{P})$ with information cost $I$, if $\ic(\Pi)\le I$, $R_{det}(\Pi,\hat{v}) \le 1/4$. 
\end{definition}
%
%
 
Now we formally define the concrete questions that we are concerned with. 

 \noindent{\bf Distributed Gaussian detection problem: } 
We call the problem with $\Omega  =\{0,1\}$ and $\mathcal{P} = \{\mathcal{N}(0,\sigma^2)^n, \mathcal{N}(\delta,\sigma^2)^n\}$ the Gaussian mean detection problem, denoted by $\GD(n,m,\delta,\sigma^2)$. 
 
\noindent{\bf Distributed (sparse) Gaussian mean estimation problem: } The distributed statistical estimation problem defined by $\Omega = \mathbb{R}^d$ and $\mathcal{P} = \{\mathcal{N}(\theta,\sigma^2I_{d\times d}) : \theta\in \Omega\}$ is called the distributed Gaussian mean estimation problem, abbreviated $\GME(n,m,d,\sigma^2)$. When $\Omega = \{\theta \in \mathbb{R}^d: |\theta|_0 \le k\}$, the corresponding problem is referred to as distributed sparse Gaussian mean estimation, abbreviated $\SGME(n,m,d,k,\sigma^2)$. 

\noindent{\bf Distributed sparse linear regression: }
\noindent For simplicity and the purpose of lower bounds, we only consider sparse linear regression with a random design matrix. To fit into our framework, we can also regard the design matrix as part of the data. We have a parameter space $\Omega = \{\theta \in \mathbb{R}^d: |\theta|_0 \le k\}$.  The $j$-th data point consists of a row of design matrix $A_{j}$ and the observation $y_j = \inner{A_j,\theta} + w_j$ where $w_j\sim \mathcal{N}(0,\sigma^2)$ for $j=1,\dots, mn$, and each machine receives $n$ data points among them\footnote{We note that here for convenience, we use subscripts for samples, which is different from the notation convention used for previous problems. }. Formally, let $\mathcal{\mu}_{\theta}$ denote the joint distribution of $(A_j,y_j)$ here, and let $\mathcal{P} = \{\mu_{\theta}: \theta\in \Omega\}$. We use $\SLR(n,m,d,k,\sigma^2)$ as shorthand for this problem.

\subsection{Hellinger distance and cut-paste property}\label{subsec:cut-paste}
In this subsection, we introduce Hellinger distance, and the key property of protocols that we exploit here, the so-called ``cut-paste" property developed by~\cite{BJKS04} for proving lower bounds for set-disjointness and other problems. We also introduce some notation that will be used later in the proofs. 

\begin{definition}[Hellinger distance]\label{def:hellinger}
	Consider two distributions with probability density functions $f, g:\Omega\rightarrow \mathbb{R}$. The square of the Hellinger distance between $f$ and $g$ is defined as
	\[
	\h^2(f,g) := \frac{1}{2}\cdot \int_{\Omega} \left(\sqrt{f(x)} - \sqrt{g(x)}\right)^2 dx
	\]
\end{definition}

A key observations regarding the property of a protocol by~\cite[Lemma 16]{BJKS04} is the following: fixing $X_1 = x_1,\dots,X_m = x_m$, the distribution of $\Pi\vert_{X=x} $ can be factored in the following form, 
\begin{equation}
\Pr[\Pi = \pi \mid X = x] = p_{1,\pi}(x_1)\dots p_{m,\pi}(x_m)\label{eqn:inter5}
\end{equation}
where $p_{i,\pi}(\cdot)$ is a function that only depends on $i$ and the entire transcript $\pi$ . To see this, one could simply write the density of $\pi$ as a products of density of each messages of the machines and group the terms properly according to machines (and note that $p_{i,\pi}(\cdot)$ is allowed to depend on the entire transcript $\pi$). 

We extend equation~\eqref{eqn:inter5} to the situation where the inputs are from product distributions. For any vector $\bs{b}\in \{0,1\}^m$, let $\mu_{\bs{b}}:=\mu_{b_1}\times \dots \times \mu_{b_m}$  be a distribution over $\mathcal{X}^m$. We denote by $\Pi_{\bs{b}}$ the distribution of $\Pi(X_1,\dots,X_m)$ when $(X_1,\dots,X_m)\sim \mu_{\bs{b}}$.

%

Therefore if $X\sim \mu_{\bs{b}}$, using the fact that $\mu_{\bs{b}}$ is a product measure, we can marginalize over $X$ and obtain the marginal distribution of $\Pi$ when $X\sim \mu_{\bs{b}}$, 
\begin{equation}\label{eqn:decomp}
\Pr_{X\sim \mu_{\bs{b}}}[\Pi = \pi] 
= q_{1,\pi}(b_1)\dots q_{m,\pi}(b_m), 
\end{equation}
where $q_{j,\pi}(b_j)$ is the marginalization of $p_{j,\pi}(x)$ over $x\sim \mu_{b_j}$, that is, $q_{j,\pi}(b_j) = \int_x p_{j,\pi}(x)d\mu_{b_j}$. 

Let $\Pi_{\bs{b}}$ denote the distribution of $\Pi$ when $X\sim \mu_{\bs{b}}$. Then by the decomposition (\ref{eqn:decomp}) of $\Pi_{\bs{b}}(\pi)$ above, we have the following cut-paste property for $\Pi_{\bs{b}}$ which will be the key property of a protocol that we exploit. 

\begin{proposition}[Cut-paste property of a protocol]\label{prop:cut-paste} For any 
	$\bs{a},\bs{b}$ and $\bs{c},\bs{d}$ with $\{a_i,b_i\} = \{c_i,d_i\}$ (in a multi-set sense) for every $i\in [m]$, \begin{equation}
	\Pi_{\bs{a}}(\pi) \cdot\Pi_{\bs{b}}(\pi) =\Pi_{\bs{c}}(\pi)\cdot \Pi_{\bs{d}}(\pi)
	\end{equation}
	and therefore, \begin{equation}
	\h^2(\Pi_{\bs{a}},\Pi_{\bs{b}}) =\h^2(\Pi_{\bs{c}},\Pi_{\bs{d}})
	\end{equation}
\end{proposition}


\section{Distributed Strong Data Processing Inequalities}\label{sec:main-theorem}
In this section we prove our main Theorem~\ref{thm:main-intro}. We state a slightly weaker looking version here but in fact it implies Theorem~\ref{thm:main-intro} by symmetry. The same proof also goes through for the case when the RHS is conditioned on $V=1$.

\begin{theorem}\label{thm:main-asymmetric}
	Suppose $\mu_1\le c\cdot \mu_0$, and $\beta(\mu_0,\mu_1) = \beta$, we have 
	\begin{equation}
	\h^2(\Pi\vert_{V=0}, \Pi\vert_{V=1}) \le K(c+1)\beta\cdot\I(X;\Pi\mid V=0) \,.\label{eqn:sdpi-main-theorem}
	\end{equation}
	where $K$ is an absolute constant. 
\end{theorem}

\noindent Note that the RHS of~\eqref{eqn:sdpi-main-theorem} naturally tensorizes (by Lemma~\ref{lem:mutual-info-hellinger} that appears below) in the sense that 
\begin{equation}
\sum_{i=1}^m \I(X_i;\Pi\mid V=0)\le \I(X;\Pi\mid V=0),\label{eqn:dsdpi-inter6}
\end{equation}
since conditioned on $V = 0$, the $X_i$'s are independent. Our main idea consists of the following two steps a) We tensorize the LHS of~\eqref{eqn:sdpi-main-theorem} so that the target inequality~\eqref{eqn:sdpi-main-theorem} can be written as a sum of $m$ inequalities. b) We prove each of these $m$ inequalities using the single machine SDPI. 

To this end, we do the following thought experiment: Suppose $W$ is a random variable that takes value from $\{0,1\}$ uniformly. Suppose data $X'$ is generated as follows: $X'_j\sim \mu_W$, and for any $j\neq i$, $X_j' \sim \mu_0$. We apply the protocol on the input $X'$, and view the resulting transcript $\Pi'$ as communication between the $i$-th machine and the remaining machines. Then we are in the situation of a single machine case, that is, $W\rightarrow X_i'\rightarrow \Pi'$ forms a Markov Chain. Applying the data processing inequality~\eqref{eqn:data-processing-inequality-intro}, we obtain that 
\begin{equation}
\I(W;\Pi')\le \beta \I(X_i';\Pi') \cdot \label{eqn:dsdpi-inter8}
\end{equation}
Using Lemma~\ref{lem:info-hellinger}, we can lower bound the LHS of~\eqref{eqn:dsdpi-inter8} by the Hellinger distance and obtain
\begin{equation}
\h^2(\Pi'\vert_{W=0}, \Pi'\vert_{W=1})\le \beta \cdot\I(X_i';\Pi') \nonumber 
\end{equation}
Let $\bs{e_i} = (0,0,\dots,1,\dots,0)$ be the unit vector that only takes 1 in the $i$th entry, and $\bs{0}$ the all zero vector. Using the notation defined in Section~\ref{subsec:cut-paste}, we observe that $\Pi'\vert_{W=0}$ has distribution $\Pi_{\bs{0}}$ while  $\Pi'\vert_{W=1}$ has distribution $\Pi_{\bs{e}_i}$. Then we can rewrite the equation above as 
\begin{equation}
\h^2(\Pi_{\bs{0}}, \Pi_{\bs{e}_i})\le \beta \cdot\I(X_i';\Pi') \label{eqn:dsdpi-inter11}
\end{equation}
Observe that the RHS of~\eqref{eqn:dsdpi-inter11} is close to the first entry of the LHS of~\eqref{eqn:dsdpi-inter6} since the joint distribution of $(X_1',\Pi')$ is not very far from $X,\Pi\mid V=0$. (The only difference is that $X_1'$ is drawn from a mixture of $\mu_0$ and $\mu_1$, and note that $\mu_0$ is not too far from $\mu_1$). On the other hand, the sum of LHS of~\eqref{eqn:dsdpi-inter11} over $i\in [m]$ is lower-bounded by the LHS of~\eqref{eqn:sdpi-main-theorem}. Therefore, we can tensorize equation~\eqref{eqn:sdpi-main-theorem} into inequality~\eqref{eqn:dsdpi-inter11} which can be proved by the single machine SDPI.  
 We formalize the intuition above by the following two lemmas, 
\begin{lemma}\label{lem:mutual-info-hellinger}
	Suppose $\mu_1\le c\cdot\mu_0$, and $\beta(\mu_0,\mu_1) = \beta$,  then 
	\begin{equation}
		\h^2(\Pi_{\bs{e_i}}, \Pi_{\bs{0}})\le \frac{(c+1)\beta}{2}\cdot \I(X_i;\Pi\mid V=0)  \label{eqn:lem:mutual-info}
	\end{equation}
\end{lemma}

\begin{lemma}\label{lem:hellinger-decomposition}
	Let $\bs{0}$ be the $m$-dimensional all 0's vector, and $\bs{1}$ the all 1's vector, we have that 
	\begin{equation}
	\h^2(\Pi_{\bs{0}},\Pi_{\bs{1}})\le O(1)\cdot\sum_{i=1}^m \h^2(\Pi_{\bs{e_i}}, \Pi_{\bs{0}}) \label{eqn:hellinger-decomposition}
	\end{equation}

\end{lemma}

Using Lemma~\ref{lem:mutual-info-hellinger} and Lemma~\ref{lem:hellinger-decomposition}, we obtain Theorem~\ref{thm:main-asymmetric} straightforwardly by combining inequalities~\eqref{eqn:dsdpi-inter6}, ~\eqref{eqn:lem:mutual-info} and~\eqref{eqn:hellinger-decomposition}\footnote{Note that $\Pi_{\bs{0}}$ is the same distribution as $\Pi\vert_{V=0}$ under the notation introduced in Section~\ref{subsec:cut-paste}.}. 

Finally we provide the proof of Lemma~\ref{lem:mutual-info-hellinger}. Lemma~\ref{lem:hellinger-decomposition} is a direct corollary of Theorem~\ref{thm:hellinger-sum} (which is in turn a direct corollary of Theorem 7 of~\cite{Hellinger}) and Proposition~\ref{prop:cut-paste}. 
\begin{proof}[Proof of Lemma~\ref{lem:mutual-info-hellinger}] 
 Let $W$ be uniform Bernoulli random variable and define $X'$ and $\Pi'$ as follows: Conditioned on $W = 0$, $X'\sim \mu_{\bs{0}}$ and conditioned on $W = 1$, $X'\sim \mu_{\bs{e_i}}$. We run protocol on $X'$ and get transcript $\Pi'$. 
 
 Note that $V\rightarrow X' \rightarrow \Pi'$ is a Markov chain and so is $V\rightarrow X_i'\rightarrow \Pi'$. Also by definition, the conditional random variable $X'\vert V$ has the same distribution as the random variable $X\vert V$ in Definition~\ref{def:sdpi-intro}. Therefore by Definition~\ref{def:sdpi-intro}, we have that 
 
\begin{equation}
\beta\cdot \I(X_i';\Pi') \ge \I(V;\Pi').\label{eqn:1}
\end{equation}
It is known that mutual information can be expressed as the expectation of KL divergence, which in turn is lower-bounded by Hellinger distance. We invoke a technical variant of this argument, Lemma 6.2 of~\cite{DBLP:journals/jcss/Bar-YossefJKS04}, restated as Lemma~\ref{lem:info-hellinger}, to lower bound the right hand side. Note that $Z$ in Lemma~\ref{lem:info-hellinger} corresponds to $V$ here and $\phi_{z_1},\phi_{z_2}$ corresponds to $\Pi_{\bs{e_i}}$ and $\Pi_{\bs{0}}$. Therefore, 
\begin{equation}
	\I(V;\Pi') \ge \h ^2(\Pi_{\bs{e_i}}, \Pi_{\bs{0}}).\label{eqn:2}
\end{equation} 
It remains to relate $\I(X_i';\Pi')$ to $\I(X_i;\Pi\mid V=0)$. Note that the difference between joint distributions of $(X_i',\Pi')$ and $(X_i,\Pi)\vert_{V=0}$ is that $X_i'\sim \frac{1}{2}(\mu_0+\mu_1)$ and $X_i\vert_{V=0}\sim \mu_0$.  
We claim (by Lemma~\ref{lem:continuity}) that since $\mu_{0}\ge \frac{2}{c+1}(\frac{\mu_0+\mu_1}{2})$, we have 
\begin{equation}
	\I(X_i;\Pi\mid V=0) \ge \frac{2}{c+1}\cdot \I(X'_i;\Pi'). \label{eqn:3}
\end{equation}

Combining equations (\ref{eqn:1}), (\ref{eqn:2}) and (\ref{eqn:3}), we obtain the desired inequality. 

\end{proof}

\section{Applications to Parameter Estimation Problems}\label{sec:gaussian-mean}


\subsection{Warm-up: Distributed Gaussian mean detection}
In this section we apply our main technical Theorem~\ref{thm:main-asymmetric}  to the situation when $\mu_{0} = \mathcal{N}(0,\sigma^2)$ and $\mu_1 = \mathcal{N}(\delta,\sigma^2)$. 
We are also interested in the case when each machine receives $n$ samples from either $\mu_0$ or $\mu_1$. We will denote the product of $n$ i.i.d copies of $\mu_v$ by $\mu_v^n$, for $v\in \{0,1\}$. 

Theorem~\ref{thm:main-asymmetric} requires that a) $\beta = \beta(\mu_0,\mu_1)$ can be calculated/estimated  b) the densities of distributions $\mu_0$ and $\mu_1$ are within a constant factor with each other at every point. 

Certainly b) is not true for any two Gaussian distributions. To this end, we consider $\mu_0', \mu_1'$, the truncation of $\mu_0$ and $\mu_1$ on some support $[-\tau,\tau]$, and argue that the probability mass outside $[-\tau,\tau]$ is too small to make a difference. 

For a), we use tools provided by Raginsky~\cite{Raginsky14} to estimate the SDPI constant $\beta$. ~\cite{Raginsky14} proves that Gaussian distributions $\mu_0$ and $\mu_1$ have SDPI constant $\beta(\mu_0,\mu_1)\le O(\delta^2/\sigma^2)$, and more generally it connects the SDPI constants to transportation inequalities. We use the framework established by~\cite{Raginsky14} and apply it to the truncated Gaussian distributions $\mu_0'$ and $\mu_1'$. Our proof essentially uses the fact that $(\mu_0'+\mu_1')/2$ is a log-concacve distribution and therefore it satisfies the log-Sobolev inequality, and equivalently it also satisfies the transportation inequality. The details and connections to concentration of measures are provided in Section~\ref{subsec:truncated-gaussian}. 




\begin{theorem}\label{thm:truncated-guassian-dpi-info}
	Let $\mu_0'$ and $\mu_1'$ be the distributions obtained by truncating $\mu_0$ and $\mu_1$ on support $[-\tau,\tau]$ for some $\tau >0$. If $\delta \le \sigma$, we have $\beta(\mu_0',\mu_1') \le \delta^2/\sigma^2.$ 
\end{theorem}

As a corollary, the SDPI constant between $n$ copies of $\mu_0'$ and $\mu_1'$ is bounded by $n\delta^2/\sigma^2$. 

\begin{corollary}\label{cor:truncated-gaussian-dpi-many-samples}
	Let $\tilde{\mu}_0$ and $\tilde{\mu}_1$ be the distributions over $\mathbb{R}^n$ that are obtained by truncating $\mu_0^n$ and $\mu_1^n$ outside the ball $\mathcal{B} = \{x\in \mathbb{R}^n: |x_1+\dots +x_n|\le \tau\}$. Then when $\sqrt{n}\delta \le \sigma$, we have
	$$\beta(\tilde{\mu}_0,\tilde{\mu}_1) \le n\delta^2/\sigma^2$$
\end{corollary}

Applying our distributed data processing inequality (Theorem~\ref{thm:main-asymmetric}) on $\tilde{\mu}_0$ and $\tilde{\mu}_1$, we obtain directly that to distinguish $\tilde{\mu}_0$ and $\tilde{\mu}_1$ in the distributed setting, $\Omega\left(\frac{\sigma^2}{n\delta^2}\right)$ communication is required. By properly handling the truncation of the support, we can prove that it is also true with the true Gaussian distribution. 
\begin{theorem}\label{thm:multi-dpi-gaussian}
Any protocol estimator pair $(\Pi,\hat{v})$ that solves the distributed Gaussian mean detection problem $\GD(n,m,\delta,\sigma^2)$ with $\delta\le \sigma/\sqrt{n}$ requires communication cost and minimum information cost at least,   
	$$\Exp[|\Pi|]\ge \mic(\Pi) \ge \Omega\left(\frac{\sigma^2}{n\delta^2}\right)\,.$$
\end{theorem}


\begin{remark}
	The condition $\delta\le \sigma/\sqrt{n}$ captures the interesting regime. When $\delta \gg \sigma/\sqrt{n}$, a single machine can even distinguish $\mu_0$ and $\mu_1$ by its local $n$ samples. 
\end{remark}

\begin{proof}[Proof of Theorem~\ref{thm:multi-dpi-gaussian}]
	Let $\Pi_{\bs{0}}$ and $\Pi_{\bs{1}}$ be the distribution of $\Pi\vert V=0$ and $\Pi\vert V=1$ as defined in Section~\ref{subsec:cut-paste}. Since $\hat{v}$ solves the detection problem, we have that $\TV{\Pi_{\bs{0}} - \Pi_{\bs{1}}}\ge 1/4$. It follows from Lemma~\ref{lem:hellinger-tv} that $\h(\Pi_{\bs{0}}, \Pi_{\bs{1}}) \ge \Omega(1)$. 
	
	
	
	
	We pick a threshold $\tau = 20 \sigma$, and let $\mathcal{B} = \{z\in \mathbb{R}^n: |z_1+\dots+z_n|\le \sqrt{n}\tau\}$.  
	Let $F = 1$ denote the event that $X = (X_1,\dots,X_n)\in \mathcal{B}$, and otherwise $F= 0$. Note that $\Pr[F=1] \ge 0.95$ and therefore even if we conditioned on the event that $F=1$, the protocol estimator pair should still be able to recover $v$ with good probability in the sense that 
	\begin{equation}
	\Pr[\hat{v}(\Pi(X)) = v \mid V = v, F=1] \ge 0.6
	\end{equation}
	
	We run our whole argument conditioning on the event $F = 1$. First note that for any Markov chain $V\rightarrow X \rightarrow \Pi$, and any random variable $F$ that only depends on $X$, the chain $V\vert_{F=1}\rightarrow X\vert_{F=1} \rightarrow \Pi\vert_{F=1}$ is also a Markov Chain. 
	Second, the channel from $V$ to $X\vert_{F=1}$ satisfies that random variable $X\vert_{V=v, F=1}$ has the distribution $\tilde{\mu}_v$ as defined in the statement of Corollary~\ref{cor:truncated-gaussian-dpi-many-samples}. Note that by Corollary~\ref{cor:truncated-gaussian-dpi-many-samples}, we have that $\beta(\tilde{\mu}_0,\tilde{\mu}_1) \le n\delta^2/\sigma^2$. Also note that by the choice of $\tau$ and the fact that $\delta \le O(\sigma/\sqrt{n})$, we have that for any $z\in \mathcal{B}$, $\tilde{\mu}_0(z)\le O(1)\cdot \tilde{\mu}_1(z)$. 
	
	Therefore we are ready to apply Theorem~\ref{thm:main-asymmetric}  and conclude that 
	$$\I(X;\Pi\mid V=0, F=1)\ge \Omega(\beta(\tilde{\mu}_0,\tilde{\mu}_1)^{-1}) = \Omega(\frac{\sigma^2}{n\delta^2})$$
	
	Note that $\Pi$ is independent with $F$ conditioned on $X$ and $V=0$. Therefore we have that 
	
	$$\I(X;\Pi\mid V=0)\ge \I(X;\Pi\mid F, V=0) \ge \I(X;\Pi\vert F=1,V=0)\Pr[F=1\mid V=0] = \Omega(\frac{\sigma^2}{n\delta^2}).$$

	Note that by construction, it is also true that $\tilde{\mu}_0 \le O(1)\tilde{\mu}_1$, and therefore if we switch the position of $\tilde{\mu}_0, \tilde{\mu}_1$ and run the argument above we will have 
	$$\I(X;\Pi\mid V=1)= \Omega(\frac{\sigma^2}{n\delta^2})$$
	Hence the proof is complete. 
	
	%
	%
	
\end{proof}


\subsection{Sparse Gaussian mean estimation}

In this subsection, we prove our lower bound for the sparse Gaussian mean estimation problem via a variant of the direct-sum theorem of~\cite{DBLP:conf/nips/GargMN14} tailored towards sparse mean estimation. 

Our general idea is to make the following reduction argument: Given a protocol $\Pi'$ for $d$-dimensional $k$-sparse estimation problem with information cost $I$ and loss $R$, we can construct a protocol $\Pi'$ for the detection problem with information cost roughly $I/d$ and loss $R/k$. The protocol $\Pi'$ embeds the detection problem into one random coordinate of the $d$-dimensional problem, prepares fake data on the remaining coordinates, and then runs the protocol $\Pi$ on the high dimensional problem. It then extracts information about the true data from the corresponding coordinate of the high-dimensional estimator. 

The key distinction from the construction of~\cite{DBLP:conf/nips/GargMN14} is that here we are not able to show that $\Pi'$ has small information cost, but only able to show that $\Pi'$ has a small minimum information cost \footnote{This might be inevitable because protocol $\Pi$ might reveal a lot information for the nonzero coordinate of $\theta$ but since there are very few non-zeros, the total information revealed is still not too much.}. This is the reason why in Theorem~\ref{thm:multi-dpi-gaussian} we needed to bound the minimum information cost instead of the information cost.

To formalize the intuition, let $\mathcal{P} = \{\mu_0,\mu_1\}$ define the detection problem.  Let $\Omega_{d,k,\delta} = \{\theta: \theta\in \{0,\delta\}^d, |\theta|_0 \le k\}$ and $\mathcal{Q}_{d,k,\delta} = \{\mu_{\theta} = \mu_{\theta_1/\delta}\times \dots \times \mu_{\theta_d/\delta}: \theta\in \Omega_{d,k,\delta}\}$. Therefore $\mathcal{Q}$ is a special case of the general $k$-sparse high-dimensional problem. We have that

\begin{theorem}[Direct-sum for sparse parameters] \label{thm:direct_sum_sparse}
Let $d \ge 2k$, and $\mathcal{P}$ and $\mathcal{Q}$ defined as above. 
If there exists a protocol estimator pair $(\Pi, \hat{\theta})$ that solves the detection task $\Task(n,m,\cal{Q})$ with information cost $I$ and mean-squared loss $R\le \frac{1}{16}k\delta^2$, then there exists a protocol estimator pair $(\Pi',\hat{v}')$ (shown in Protocol~\ref{Protocol1} in Section~\ref{sec:direct-sum}) that solves the task $\Taskdet(n, m,\cal{P})$ with minimum information cost $\frac{I}{d-k+1}$. 
\end{theorem} 

The proof of the theorem is deferred to Section~\ref{sec:direct-sum}. Combining Theorem~\ref{thm:multi-dpi-gaussian} and Theorem \ref{thm:direct_sum_sparse}, we get the following theorem:

\begin{theorem}\label{thm:sparse-gaussian-mean}
	Suppose $d \ge 2k$. Any protocol estimator pair $(\Pi,\hat{v})$ that solves the $k$-sparse Gaussian mean problem $\SGME(n,m,d,k,\sigma^2)$ with mean-squared loss $R$ and information cost $I$ and communication cost $C$ satisfy that 
	\begin{equation}
		R \ge \Omega\left(\min\left\{\frac{\sigma^2k}{n}, \max \left\{\frac{\sigma^2dk}{nI}, \frac{\sigma^2k}{nm} \right\}\right\}\right)\ge \Omega\left(\min\left\{\frac{\sigma^2k}{n}, \max \left\{\frac{\sigma^2dk}{nC}, \frac{\sigma^2k}{nm} \right\}\right\}\right)\,.\label{eqn:main}
	\end{equation}
\end{theorem}

\noindent Intuitively, to parse equation~\eqref{eqn:main}, we remark that the term $\frac{\sigma^2 k}{n}$ comes from the fact that any local machine can achieve this error $O(\frac{\sigma^2 k}{n})$ using only its local samples, and the term $\frac{\sigma^2k}{nm}$ is the minimax error that the machines can achieve with infinite amount of communication.  When the target error is between these two quantities, equation~\eqref{eqn:main} predicts that the minimum communication $C$ should scale inverse linearly in the error $R$. 

Our theorem gives a tight tradeoff between $C$ and $R$ up to logarithmic factor, since it is known~\cite{DBLP:conf/nips/GargMN14} that for any  communication budget $C$, there exists protocol which uses $C$ bits and has error $R\le O\left(\min\left\{\frac{\sigma^2k}{n}, \max \left\{\frac{\sigma^2dk}{nC}, \frac{\sigma^2k}{nm} \right\}\right\}\cdot\log d\right)$.

As a side product, in the case when $k= d/2$, our lower bound improves previous works~\cite{DBLP:journals/corr/DuchiJWZ14} and~\cite{DBLP:conf/nips/GargMN14} by a logarithmic factor, and turns out to match the upper bound in~\cite{DBLP:conf/nips/GargMN14} up to a constant factor. 

\begin{proof}[Proof of Theorem~\ref{thm:sparse-gaussian-mean}]
	If $R \le \frac{1}{16}\frac{k\sigma^2}{n}$ then we are done. Otherwise, let $\delta  := \sqrt{16R/k}\le \sigma/\sqrt{n}$. Let $\mu_0 = \mathcal{N}(0,\sigma^2)$ and $\mu_1 = \mathcal{N}(\delta,\sigma^2)$ and $\mathcal{P} = \{\mu_0,\mu_1\}$. Let  $\mathcal{Q}_{d,k,\delta} = \{\mu_{\theta} = \mu_{\theta_1/\delta}\times \dots \times \mu_{\theta_d/\delta}: \theta\in \Omega_{d,k,\delta}\}$. Then $\Task(n,m,\mathcal{Q})$ is just a special case of sparse Gaussian mean estimation problem $\SGME(n,m,d,k,\sigma^2)$, and $T(n,m,\mathcal{P})$ is the distributed Gaussian mean detection problem $\GD(n,m,\delta,\sigma^2)$. Therefore, by Theorem~\ref{thm:direct_sum_sparse}, there exists $(\Pi',\hat{v}')$ that solves  $\GD(n,m,\delta,\sigma^2)$ with minimum information cost $I' = O(I/d)$. Since $\delta \le O(\sigma/\sqrt{n})$, by Theorem~\ref{thm:multi-dpi-gaussian} we have that $I'\ge \Omega(\sigma^2/(n\delta^2))$. It follows that $I \ge \Omega(d\sigma^2/(n\delta^2)) = \Omega(kd\sigma^2/(nR))$. To derive~\eqref{eqn:main}, we observe that $\Omega(\sigma^2k/{nm})$ is the minimax lower bound for $R$, which completes the proof. 
\end{proof}

To complement our lower bounds, we also give a new protocol
for the Gaussian mean estimation problem achieving communication
optimal up to a constant factor in any number of dimensions in the dense case. 
Our protocol is a {\it simultaneous protocol}, whereas
the only previous protocol achieving optimal communication
requires $\Omega(\log m)$ rounds \cite{DBLP:conf/nips/GargMN14}. 
This resolves an open question in Remark 2 of \cite{DBLP:conf/nips/GargMN14},  
improving the trivial protocol in which 
each player sends its truncated Gaussian to the coordinator by an $O(\log m)$ factor.  


\begin{theorem}\label{thm:one-way-upper-bound}
	For any $0\le \alpha\le 1$, there exists a protocol that uses one round of communication for the Gaussian mean estimation problem $\GME(n,m,d,\sigma^2)$ with communication cost $C = \alpha dm$ and mean-squared loss $R = O\left(\frac{\sigma^2d}{\alpha mn}\right)$. 
\end{theorem}
The protocol and proof of this theorem are deferred to Section~\ref{sec:one-way}, though we
mention a few aspects here. We first give a protocol under the assumption
that $|\theta|_{\infty}\le \frac{\sigma}{\sqrt{n}}$. The protocol trivially generalizes to $d$ dimensions
so we focus on $1$ dimension. The protocol coincides 
with the first round of the multi-round protocol
in \cite{DBLP:conf/nips/GargMN14}, yet we can extract all necessary information
in only one round, by having each machine send a single bit indicating if its
input Gaussian is positive or negative. Since the mean is on the same order as the standard deviation, one
can bound the variance and give an estimator based on the Gaussian density function. 
In Section \ref{sec:generalUpper} the mean of the Gaussian is allowed to be
much larger than the variance, and this no longer works. Instead, a few machines send their truncated
inputs so the coordinator learns a crude approximation. To refine this approximation, in parallel 
the remaining machines each send 
a bit which is $1$ with probability $x - \lfloor x \rfloor$, where $x$ is the machine's 
input Gaussian. This can be
viewed as rounding a sample of the ``sawtooth wave function'' $h$ applied to a Gaussian.
For technical reasons each machine needs to send two bits, another which is $1$ with probability
$(x+1/5) - \lfloor (x+1/5) \rfloor$. We give an estimator based on an analysis using the Fourier series
of $h$. 

\paragraph{Sparse Gaussian estimation with signal strength lower bound}
Our techniques can also be used to study the optimal rate-communication tradeoffs in the presence of a strong signal in the non-zero coordinates, which is sometimes assumed for sparse signals. That is, suppose the machines are promised that the mean $\theta \in \mathcal{R}^d$ is $k$-sparse and also if $\theta_i \ne 0$, then $|\theta_i| \ge \eta$, where $\eta$ is a parameter called the signal strength. We get tight lower bounds for this case as well.  

\begin{theorem}{\label{gaussian_signal}}
	For  $d \ge 2k$ and $\eta^2 \ge 16 R/k$, any protocol estimator pair $(\Pi,\hat{v})$ that solves the $k$-sparse Gaussian mean problem $\SGME(n,m,d,k,\sigma^2)$ with signal strength $\eta$ and mean-squared loss $R$ requires information cost (and hence expected communication cost) at least $\Omega \left( \frac{\sigma^2 d}{n \eta^2} \right)$. 
\end{theorem}
Note that there is a protocol for $\SGME(n,m,d,k,\sigma^2)$ with signal strength $\eta$ and mean-squared loss $R$ that has communication cost $\tilde{O}\left(\min\left\{ \frac{\sigma^2 d}{n \eta^2} + \frac{\sigma^2 k^2}{n R} , \frac{\sigma^2 d k}{n R}\right\}\right)$. In the regime where $\eta^2 \ge 16 R/k$, the first term dominates and by Theorem \ref{gaussian_signal}, and the fact that $\frac{\sigma^2 k^2}{n R}$ is a lower bound even when the machines know the support \cite{DBLP:conf/nips/GargMN14}, we also get a matching lower bound. In the regime where $\eta^2 \le 16 R/k$, second term dominates and it is a lower bound by Theorem \ref{thm:sparse-gaussian-mean}. 

\begin{proof}[Proof of Theorem~\ref{gaussian_signal}]
	The proof is very similar to the proof of Theorem \ref{thm:direct_sum_sparse}. Given a protocol estimator pair $(\Pi,\hat{v})$ that solves $\SGME(n,m,d,k,\sigma^2)$ with signal strength $\eta$, mean-squared loss $R$ and information cost $I$ (where $\eta^2 \ge 16 R/k$), we can find a protocol $\Pi'$ that solves the Gaussian mean detection problem $\GD(n.m,\eta,\sigma^2)$ with information cost $\le O(I/d)$ (as usual the information cost is measured when the mean is $0$). $\Pi'$ would be exactly the same as Protocol \ref{Protocol1} but with $\mu_0$ replaced by $\mathcal{N}(0,\sigma^2)$, $\mu_1$ replaced by $\mathcal{N}(\eta,\sigma^2)$ and $\delta$ replaced by $\eta$. We leave the details to the reader.
\end{proof}

\subsection{Lower bound for Sparse Linear Regression}\label{sec:sparse-linear-regression}

In this section we consider the sparse linear regression problem $\SLR(n,m,d,k,\sigma^2)$ in the distributed setting as defined in Section~\ref{sec:setup}. 
Suppose the $i$-th machine receives a subset $S_i$ of the $mn$ data points, and we use $A_{S_i}\in \mathbb{R}^{n\times d}$ to denote the design matrix that the $i$-th machine receives and $y_{S_i}$ to denote the observed vector. That is, 
$
y_{S_i}= A_{S_i}\theta + w_{S_i}, 
$
where $w_{S_i}\sim \mathcal{N}(0,\sigma^2I_{n\times n})$ is Gaussian noise.  

This problem can be reduced from the sparse Gaussian mean problem, and thus its communication can be lower-bounded. It follows straightforwardly from our Theorem~\ref{thm:sparse-gaussian-mean} and the reduction in Corollary 2 of~\cite{DBLP:journals/corr/DuchiJWZ14}. 
To state our result, we assume that the design matrices $A_{S_i}$ have uniformly bounded spectral norm $\lambda\sqrt{n}$. That is, 
$
\lambda = \max_{1\le i\le m} \|A_{S_i}\|/\sqrt{n}. 
$
\begin{corollary}\label{cor:sparse-linear-regression}
	Suppose machines receive data from the sparse linear regression model. Let $\lambda$ be as defined above. If there exists a protocol under which the machines can output an estimator $\hat{\theta}$ with mean squared loss $R = \Exp[\|\hat{\theta}-\theta\|^2]$ with communication $C$, then 
	$R\cdot C \ge \Omega(\frac{\sigma^2kd}{\lambda^2 n})$.
\end{corollary}
When $A_{S_i}$ is a Gaussian design matrix, that is, the rows of $A_{S_i}$ are i.i.d drawn from distribution $\mathcal{N}(0,I_{d\times d})$, we have $\lambda = O\left(\max\{\sqrt{d/n}, 1\}\right)$ and Corollary~\ref{cor:sparse-linear-regression} implies that to achieve the statistical minimax rate $R = O(\frac{k\sigma^2}{nm})$, the algorithm has to communicate $\Omega(m\cdot \min\{n,d\})$ bits. The point is that we get a lower bound that doesn't depend on $k$-- that is, with sparsity assumptions, it is impossible to improve both the loss and communication so that they depend on the intrinsic dimension $k$ instead of the ambient dimension $d$. Moreover, in the regime when $d/n \rightarrow c$ for a constant $c$, our lower bound matches the upper bound of \cite{LSQT15} up to a logarithmic factor. The proof follows Theorem~\ref{thm:sparse-gaussian-mean} and the reduction from Gaussian mean estimation to sparse linear regression of~\cite{ZDJW13} straightforwardly and is deferred to Section~\ref{sec:app:thm:multi-dpi-gaussian}. 

\section{Direct-sum Theorem for Sparse Parameters}\label{sec:direct-sum}
\begin{Protocol}
	Unknown parameter: $v \in \{0,1\}$  \\
	Inputs: Machine $j$ gets $n$ samples $X_j = (X_j^{(1)}, \ldots, X_j^{(n)})$, where $X_j$ is distributed according to $\mu_{v}^n$. 
	\begin{enumerate}
		\item All machines publicly sample $k$ independent coordinates $I_1,\ldots,I_k\subset[d]$ (without replacement).
		\item Each machine $j$ locally prepares data $\widetilde{X}_j = \left(\widetilde{X}_{j,1},\dots,\widetilde{X}_{j,d}\right)$ as follows: The $I_1$-th coordinate is embedded with the true data, $\wtX_{j,I_1} = X_j$. For $r = 2,\dots, k$, $j$-th the machine draws $\wtX_{j,I_r}$ privately from distribution $\mu_1^n$. For any coordinate $i\in [d]\backslash \{I_1,\dots,I_k\}$, the $j$-th machine draws privately $\wtX_{j,i}$ from the distribution $\mu_0^n$. 
		\item The machines run protocol $\Pi$ with input data $\wtX$. 
		\item If $|\hat{\theta}(\Pi)_{I_1}| \ge \delta/2$, then the machines output $1$, otherwise they output $0$. 
	\end{enumerate}
	\caption{direct-sum reduction for sparse parameter}
	\label{Protocol1}
\end{Protocol}

We prove Theorem~\ref{thm:direct_sum_sparse} in this section. 
	Let $\Pi'$ be the protocol described in Protocol~\ref{Protocol1}. Let $\theta\in \mathbb{R}^d$ be such that $\theta_{I_1} = v\delta$ and $\theta_{I_r} = \delta$ for $r = 2,\dots,k$, and $\theta_i = 0$ for $i\in [d]\backslash \{I_1,\dots,I_k\}$. We can see that by our construction, the distribution of $\wtX_j$ is the same as $\mu_{\theta}^n$, and all $X_j$'s are independent. Also note that $\theta$ is $k$-sparse.  Therefore when $\Pi'$ invokes $\Pi$ on data $\wtX$, $\Pi$ will have loss $R$ and information cost $I$ with respect to $\wtX$. 

	We first verify that the protocol $\Pi$ does distinguish between $v = 0$ and $v = 1$. 
	\begin{proposition}\label{prop:direct-sum-dection}
		Under the assumption of Theorem~\ref{thm:direct_sum_sparse}, when $v = 1$, we have that 
		\begin{align}
		\mathbb{E} \left[ |\hat{\theta}(\Pi)_{I_1} - \delta|^2 \right] \le \frac{R}{k}\label{eqn:case-1}
		\end{align}
		and when $v =0$, we have 	
		\begin{align}
		\mathbb{E} \left[ |\hat{\theta}(\Pi)_{I_1}|^2 \right] \le \frac{R}{d-k+1}\label{eqn:direct-sum-eqn1}
		\end{align}
		Moreover, with probability at least 3/4, $\Pi'$ outputs the correct answer $v$. 
	\end{proposition}
	\begin{proof}

	We know that $\Pi$ has mean-squared loss $R$, that is, 
	\begin{align*}
	R((\Pi, \hat{\theta}), \theta) &= \mathbb{E} \left[ ||\hat{\theta}(\Pi) - \theta ||_2^2\right] \\
	&= \mathbb{E} \left[ \sum_{i=1}^d |\hat{\theta}(\Pi)_i - \theta_i|^2\right]
	\end{align*}
	Here the expectation is over the randomness of the protocol $\Pi$ and randomness of the samples $\wtX_1,\ldots,\wtX_m$. 
We first prove equation~\eqref{eqn:direct-sum-eqn1}, that is
	\begin{align*}
		\mathbb{E} \left[ |\hat{\theta}(\Pi)_{I_1}|^2 \right] \le \frac{R}{d-k+1}
	\end{align*}
	Here the expectation is over $I_1,\ldots,I_k$ in addition to being over the randomness of $\Pi$ and the samples $\wtX_1,\ldots,\wtX_m$. We will in fact prove this claim for any fixing of $I_2,\ldots,I_k$ to some $i_2,\ldots,i_k$. Then $I_1$ is a random coordinate in $[d] \backslash \{i_2,\dots,i_k\}$. Then
	
	\begin{align*}
	\mathbb{E} \left[ |\hat{\theta}(\Pi)_{I_1}|^2 \mid I_r = i_r, r\ge 2\right] &= \frac{1}{d-k+1} \sum_{i \in [d] \backslash \{i_2,\ldots,i_k\}} \mathbb{E} \left[ |\hat{\theta}(\Pi)_{i}|^2 \mid I_r = i_r, r\ge 2\right] \\
	&\le \frac{1}{d-k+1} \left( \sum_{i \in [d] \backslash \{i_2,\ldots,i_k\}} \mathbb{E} \left[ |\hat{\theta}(\Pi)_{i}|^2 \mid I_r = i_r, r\ge 2 \right] \right.\\
	&\left.+ \sum_{i \in \{i_2,\ldots,i_k\}} \mathbb{E} \left[ |\hat{\theta}(\Pi)_{i} - \delta|^2 \mid I_r = i_r, r\ge 2\right] \right) \\
	\end{align*} 
	Taking expectation over $I_2,\dots,I_r$ we obtain
	\begin{align*}
	\mathbb{E} \left[ |\hat{\theta}(\Pi)_{I_1}|^2\right] \le \frac{1}{d-k+1}\sum_{i=1}^d \mathbb{E} \left[ |\hat{\theta}(\Pi)_{i} - \theta|^2\right]
	&= \frac{1}{d-k+1} R((\Pi, \hat{\theta}), \theta) \\
	&\le \frac{R}{d-k+1}
	\end{align*}
	
	In order to prove equation~\eqref{eqn:case-1}, 
	we prove the statement for every fixing of $\{I_1,\ldots,I_k\}$ to some $S \subset [d]$. 
	
	\begin{align*}
	&\mathbb{E} \left[ |\hat{\theta}(\Pi)_{I_1} - \delta|^2 \mid \{I_1,\ldots,I_k\} = S\right] \\ &= \frac{1}{k} \sum_{i \in S} \mathbb{E} \left[ |\hat{\theta}(\Pi)_{i} - \delta|^2 \mid \{I_1,\ldots,I_k\} = S\right] \\
	&\le \frac{1}{k} \left(  \sum_{i \in S} \mathbb{E} \left[ |\hat{\theta}(\Pi)_{i} - \delta|^2\mid \{I_1,\ldots,I_k\} = S \right] + \sum_{i \notin S} \mathbb{E} \left[ |\hat{\theta}(\Pi)_{i}|^2 \mid \{I_1,\ldots,I_k\} = S\right] \right) \\
	& = \frac{1}{k} \sum_{i=1} ^d \mathbb{E} \left[ |\hat{\theta}(\Pi)_{i} - \delta|^2\mid \{I_1,\ldots,I_k\} = S \right] \\
	\end{align*} 
	Taking expectation over $I_1,\dots,I_k$ we obtain, 
	\begin{align*}
	\Exp\left[\mathbb{E} \left[ |\hat{\theta}(\Pi)_{I_1} - \delta|^2\right] \mid \{I_1,\ldots,I_k\} = S\right] &= \frac{1}{k} R((\Pi, \hat{\theta}), \theta) \le \frac{R}{k} 
	\end{align*}
	The last statement of proposition follows easily from Markov's inequality and the assumption that $R\le k\delta^2/16$. 
		\end{proof}
	Now we prove the information cost of the protocol $\Pi'$ under the case  $v = 0$ is small. 
	
	\begin{proposition}
		Under the assumption of Theorem~\ref{thm:direct_sum_sparse}, we have 
		$$\mic(\Pi') \le \I_0(\Pi' ; X_1,\ldots,X_m\mid \pub')\le \frac{I}{d-k+1}$$ where $X_j \sim \mu_0^n$ and  $\pub'$ is the public coin used by $\Pi'$. 
	\end{proposition}
	
	\begin{proof}
		
	Let us denote $\left( \wtX_{j,i}^{(1)},\dots,\wtX_{j,i}^{(n)} \right)$ by $\wtX_{j,i}$, that is, $\wtX_{j,i}$ is the collection of $i$-th coordinates of the samples on machine $j$. Let $\pub$ be the public coins used by protocol $\Pi$. Note that $\pub'$ are just $I_1,\dots,I_k$ and $\pub$, therefore, the information cost of $\Pi'$ is  
	\begin{align}
	\I_0(\Pi' ; X_1,\ldots,X_m\mid \pub') &= \I(\Pi; \wtX_{1,I_1}, \ldots, \wtX_{m,I_1} | I_1,\ldots,I_k,\pub) \nonumber\\
	&= \Exp_{i_2, \ldots, i_k} \left[I(\Pi; \wtX_{1,I_1}, \ldots, \wtX_{m,I_1} | I_1,I_2 = i_2, \ldots, I_k = i_k,\pub)\right]\label{eqn:direct-sum-3}
	\end{align}
	For each $i_2,\ldots,i_k$, we will prove that $\I(\Pi; \wtX_{1,I_1}, \ldots, \wtX_{m,I_1} | I_1,I_2 = i_2, \ldots, I_k = i_k,\pub) \le I/(d-k+1)$. Note that conditioned on $I_r = i_r$ for $r\ge 2$, $I_1$ is uniform over $[d]\backslash\{i_2,\dots,i_k\}$
	\begin{align}
	&\I(\Pi; \wtX_{1,I_1}, \ldots, \wtX_{m,I_1} | I_1,I_2 = i_2, \ldots, I_k = i_k,\pub)\\ &= \frac{1}{d-k+1} \sum_{i \in [d] \backslash \{i_2,\ldots,i_k\}} \I(\Pi; \wtX_{1,i}, \ldots, \wtX_{m,i} | I_1 = i,I_2 = i_2, \ldots, I_k = i_k,\pub)\nonumber\\
	&= \frac{1}{d-k+1} \sum_{i \in [d] \backslash \{i_2,\ldots,i_k\}} \I(\Pi; \wtX_{1,i}, \ldots, \wtX_{m,i} | I_2 = i_2, \ldots, I_k = i_k,\pub)\nonumber\\
	&\le  \frac{1}{d-k+1} \I \left(\Pi; \left( \wtX_{1,i}, \ldots, \wtX_{m,i} \right)_{i \in [d] \backslash \{i_2,\ldots,i_k\}} | I_2 = i_2, \ldots, I_k = i_k,\pub\right) \nonumber\\
	&\le \frac{1}{d-k+1} \I(\Pi; \wtX_1,\ldots,\wtX_m | I_2 = i_2, \ldots, I_k = i_k,\pub) \label{eqn:direct-sum-eqn2}
	\end{align}
	The second equality follows from the fact that the distribution of $\wtX_{1,i}, \ldots, \wtX_{m,i}$ for \linebreak $i \in [d] \backslash \{i_2,\ldots,i_k\}$ does not depend on $i$ and the protocol $\Pi$ is also oblivious of $I_1$ and hence we can remove the conditioning on $I_1=i$. First inequality follows from lemma \ref{lem:mutual-info-additivity} and the fact that $\wtX_{1,i}, \ldots, \wtX_{m,i}$ are independent across $i$. The second inequality follows from the fact that $I(A; B) \le I(A;B, C)$.
	
	Finally, note that $\Pi$ performs the task $\Task(n,m,\cal{Q})$ with information cost $I = \sup_{\theta}\I_{\theta}(\Pi; \tilde{X}\mid \pub)$. Note that conditioned on $I_r=i_r$ and $I_1 = i$, $\wtX$ are drawn from some valid $\mu_{\theta}$ with a $k$-sparse $\theta$. Therefore by the definition of information cost, we have that
	\begin{equation}
	\I(\Pi; \wtX_1,\ldots,\wtX_m | I_1 = i,I_2 = i_2, \ldots, I_k = i_k,\pub) \le I\label{eqn:direct-sum-4}
	\end{equation} 
	
	Hence it follows from equations~\eqref{eqn:direct-sum-3} and~\eqref{eqn:direct-sum-eqn2} and~\eqref{eqn:direct-sum-4}, we have that 
	
	\begin{equation}
	\I_0(\Pi' ; X_1,\ldots,X_m\mid \pub')\le \frac{I}{d-k+1}
	\end{equation}
	ant it follows by definition that $\mic(\Pi')\le \frac{I}{d-k+1}$. 
\end{proof}

\section{Data Processing Inequality for Truncated Gaussian}\label{sec:data-processing-ineq}

%
%
%

In this section, we prove Theorem~\ref{thm:truncated-guassian-dpi-info}, the SDPI for truncated guassian distributions. We first survey the connection between SDPI and transportation inequalities established by Raginsky~\cite{Raginsky14} in Section~\ref{subsec:general_bound}. Then we prove in Section~\ref{subsec:log-concave} that when a distribution has log-concave density function on a finite interval, it satisfies the transportation inequalities. These preparations imply straightforwardly Theorem~\ref{thm:truncated-guassian-dpi-info}, which is proved in Section~\ref{subsec:truncated-gaussian}. 

\subsection{SDPI Constant and Transportation Inequality}\label{subsec:general_bound}


Usually in literature, the inequality  \eqref{eqn:data-processing-inequality-intro} is referred to SDPI for mutual information. Here we introduce the more common version of strong data processing inequality, which turns out to be generally equivalent to SDPI for mutual information. 

\begin{lemma} \label{lem:sdpi-info-spdi}
	Consider the joint distribution of $(V, X)$ where $V \sim B_{1/2}$ and conditioned on $V=v$, we have $X \sim \mu_v$. Note that $X$ is distributed  according to the distribution $\mu = (\mu_0+\mu_1)/2$. By Bayes' rule, we can define the reverse channel $K: X\rightarrow V$ with transition probabilities $\{K(v|x): v\in \{0,1\}, x\in \mathbb{R}\}$ the same as the conditional probabilities $P_{V|X}$ of the above joint distribution. For any distribution $\nu$ over $\mathbb{R}$, let $\nu K$ denote the distribution of the output $v$ of $K$ if the input $x$ is distributed according to $\nu$. Then
	\begin{equation}
	\beta(\mu_0,\mu_1) = \sup_{\nu \neq \mu } \frac{\KL(\nu K\|\mu K)}{\KL(\nu\| \mu )} \label{eqn:sdpi-info-spdi}
	\end{equation}
\end{lemma}


Thus, it suffices to bound from above the RHS of \eqref{eqn:sdpi-info-spdi}. We use the technique developed in Theorem 3.7 of~\cite{Raginsky14}, which relates the strong data processing inequality with the concentration of measure and specifically the transportation inequality. 

To state the transportation inequality, we define the Wasserstein distance $w_1(\cdot,\cdot)$ between two probability measures, 

\begin{definition}
	The $w_1$ distance between two probability measure $\mu,\nu$ over $\mathbb{R}$ is defined as 
	\begin{equation}
	w_1(\nu,\mu) = \sup_{f: f \textrm{ is  1-Lipschitz}} \left|\int fd\nu - \int f d\mu \right|
	\end{equation}
\end{definition}

We will prove a simple transportation inequality relates the cost of transporting $\nu$ to $\mu$ in Wasserstein distance $w_1$ with the KL-divergence between $\nu$ and $\mu$, 

\begin{equation}\label{eqn:transportation}
w_1(\nu,\mu)^2\le \alpha \KL(\nu\|\mu). 
\end{equation}
for a certain value of $\alpha$ in section \ref{subsec:log-concave}.
For a complete survey of transportation inequalities with other cost functions, please see the survey of Gozlan and L{\'e}onard~\cite{gozlan2010transport}. However, before proving the transportation inequality, we show how to use it to derive a bound on $\beta(\mu_0, \mu_1)$.

\begin{lemma}[A special case of Theorem 3.7~\cite{Raginsky14}]\label{lem:specical_case_Raginsky}
	Suppose for any $v\in \{0,1\}$, $f_v(x) = \Pr[V = v\mid X =x ]$ is $L$-Lipschitz, and transportation inequality \eqref{eqn:transportation} is true for $\mu = (\mu_0+\mu_1)/2$ and any measure $\nu$, then 
		\begin{equation}
		\beta(\mu_0,\mu_1) = \sup_{\nu \neq \mu } \frac{\KL(\nu K\|\mu K)}{\KL(\nu\| \mu )} \le \alpha L^2
		\end{equation}
\end{lemma}

\begin{proof}[Proof of Lemma~\ref{lem:specical_case_Raginsky}]
	We basically follow the proof of Theorem 3.7 of~\cite{Raginsky14} with some simplifications and modifications. Note $\mu K$ is the unbiased Bernoulli distribution and by the fact that  KL divergence is not greater than $\chi^2$ distance, we have 
	\begin{align}
	\KL(\nu K\| \mu K) & \le \chi^2(\nu K\| \mu K)=\sum_{v\in \{0,1\}}\frac{(\mu K(v) - \nu K(v))^2}{\mu K(v)}\nonumber\\
	&  =2\sum_{v\in \{0,1\}}(\mu K(v) - \nu K(v))^2 \label{eqn:inter7}
	\end{align}
	
	Fixing any $v\in \{0,1\}$, we have that 
	\begin{eqnarray}
	\left|\mu K(v) - \nu K(v)\right|  &=& \left|\int \Pr[V=v\mid X=x] d\mu - \int \Pr[V=v\mid X=x] d\nu\right| \nonumber\\
	&=& \left|\int f_v(x) d\mu - \int f_v(x)d\nu\right| \nonumber\\
	&\le & Lw_1(\nu,\mu)\label{eqn:inter6}
	\end{eqnarray}
	
	where the last inequality is by the definition of Wasserstein distance and the fact that $f_v(x)$ is $L$-Lipschitz. 
	
	It follows from \eqref{eqn:inter6} and \eqref{eqn:inter7} that 

	\begin{align*}
	\KL(\nu K\| \mu K) \le L^2 w_1^2(\nu,\mu).
	\end{align*}
	Then by transportation inequality \eqref{eqn:transportation} we have that 
	
	\begin{align*}
	\KL(\nu K\| \mu K) \le L^2 w_1^2(\nu,\mu) \le \alpha L^2 D(\nu\|\mu). 
	\end{align*}
	
\end{proof}

\subsection{Proving transportation inequality via concentration of measure}\label{subsec:log-concave}

In this subsection, we show that if $\mu$ is log-concave then it satisfies transportation inequality \eqref{eqn:transportation}. To obtain the following theorem, we use a series of tools from the theory of concentration of measures in a straightforward way, albeit that in our setting, $\mu$ has only support on a finite interval and therefore we need to take some additional care.  

\begin{theorem}\label{thm:log-concave-transportation}
	Suppose $\mu$ is a measure defined on $[a,b]$ with $d\mu = \exp(-U(x))dx$, and $\nabla^2 u(x)\ge c$, then for any measure $\nu$ we have 
	\begin{equation}
	w_1(\nu,\mu)^2 \le \frac{2}{c}\cdot \KL(\nu\|\mu). \label{eqn:transportation-theorem}
	\end{equation}	
\end{theorem}

In addition, it can be proved by direct calculation that if both $\mu_0$ and $\mu_1$ are log-concave and $\mu_0$ and $\mu_1$ are not too far away in some sense, then $\mu = (\mu_0+\mu_1)/2$ is also log-concave with similar parameters. 

\begin{lemma}\label{lem:checking-log-concave}
	Suppose distribution $\mu_0$ and $\mu_1$ has supports on $[a,b]$ with $d\mu_0 = \exp(-u_0(x))dx$ and $d\mu_1 = \exp(-u_1(x))dx$. Suppose $\nabla^2 u_0(x) \ge c$, and $\nabla^2 u_0(x) \ge c$, and $|\nabla u_0(x) - \nabla u_1(x)|\le \sqrt{2c}$ then 
	then $\mu = \frac{1}{2}(\mu_0+\mu_1)$ satisfies that $d\mu =\exp(-u(x))dx$ with $\nabla^2 u(x)\ge \frac{c}{2}$. 
\end{lemma}

To prove Theorem~\ref{thm:log-concave-transportation}, we exploit the well-established connections between transportation inequality, concentration of measure and log-Sobolev inequalities. First of all, transportation inequality \eqref{eqn:transportation-theorem} with Wasserstein $w_1$ and KL-divergence ties closely to the concentration of probability measure $\mu$.  
The theorem of Bobkov-Gotze established the exact connection: 

\begin{theorem}[Bobkov-Gotze~\cite{bobkov1999exponential} Theorem 3.1]\label{thm:Bobkov-Gotze}Let $\mu \in \mathds{P}_1$ be a probability measure on a metric space $(\mathbb{X}, d)$. Then the following two are equivalent for $X\sim \mu$. 
	\begin{enumerate}
		\item $w_1(\nu,\mu)\le \sqrt{2\sigma^2\KL(\nu\|\mu)}$ for all $\nu$.
		\item $f(X)$ is $\sigma^2$-subgaussian for every 1-Lipschitz function $f$. 
	\end{enumerate}
\end{theorem}




Using Theorem~\ref{thm:Bobkov-Gotze}, in order to prove Theorem~\ref{thm:log-concave-transportation}, it suffices to prove the concentration of measure for $f(X)$ when $X\sim \mu$, and $f$ is 1-Lipschitz. Although one might prove $f(X)$ is subgaussian directly by definition, we use the log-Sobolev inequality to get around the tedious calculation. We begin by defining the entropy of a nonnegative random variable.  

\begin{definition}
	The entropy of the a nonnegative random variable $Z$ is defined as 
	\begin{equation}
	\ent[Z] : =  \Exp[Z\log Z] - \Exp[Z]\log \Exp[Z]
	\end{equation}
\end{definition}

Entropy is very useful for proving concentration of measure. As illustrated in the following lemma, to prove $X$ is subgaussian we only need to bound $\ent[e^{\lambda X}]$ by $\Exp[e^{\lambda X}]$.
\begin{lemma}[Herbst, c.f.~\cite{Ledoux01}]\label{lem:herbst}
	Suppose that for some random variable $X$, we have 
	\begin{equation}
	\ent[e^{\lambda X}]\le \frac{\lambda^2\sigma^2}{2}\Exp[e^{\lambda X}], \quad \quad \textrm{ for all $\lambda \ge 0$}\label{eqn:entropy-exp}
	\end{equation}
	Then 
	$$\psi(\lambda) := \log \Exp[e^{\lambda(X-\Exp X)}] \le \frac{\lambda^2 \sigma^2}{2}, \quad \quad \textrm{ for all $\lambda \ge 0$}$$
	and as an immediate consequences, 
	$X$ is a $\sigma^2$-subgaussian random variable. 
\end{lemma}

Therefore by Theorem~\ref{thm:Bobkov-Gotze} and Lemma~\ref{lem:herbst}, in order to prove transportation inequality, it suffices to to upper bound $\ent_{\mu}[e^{\lambda f}]$ by $\Exp[e^{\lambda f}]$. It turns out that as long as the measure $\mu$ is log-concave, we get the concentration inequality for $f(X)$ with 1-Lipschitz function $f$.  

\begin{theorem}[Theorem 5.2 of~\cite{Ledoux01}]\label{thm:log-concave-concentration}
	Let $d\mu = e^{-U}dx$ where for some $c >0$, $\nabla^2 U(x)\ge c$ for all $x\in \mathbb{R}$. Then for all smooth function $f$ on $\mathbb{R}$, 
	$$\ent_{\mu}(f^2)\le \frac{2}{c}\int |\nabla f|^2d\mu$$
\end{theorem}

As a direct corollary, we obtain inequality \eqref{eqn:entropy-exp} that we are interested in. 

\begin{corollary}\label{cor:log-concave-concentration-for-herbst}
	Let $d\mu = e^{-U}dx$ where for some $c >0$, $\nabla^2 U(x)\ge c$ for all $x\in \mathbb{R}$. Then for all 1-Lipschitz and smooth function $f$ on $\mathbb{R}$, and any $\lambda \ge 0$, we have 
	$$\ent_{\mu}(e^{\lambda f})\le \frac{\lambda^2}{2c}\Exp[e^{\lambda f}]$$
\end{corollary}

\begin{proof}[Proof of Corollary~\ref{cor:log-concave-concentration-for-herbst}]
	Applying directly Theorem~\ref{thm:log-concave-concentration} on $e^{\lambda f/2}$ we obtain, 
			$$\ent_{\mu}[e^{\lambda f}] \le \frac{2}{c} \int |\nabla e^{\lambda f/2}|^2d\mu =  \frac{2}{c} \int |e^{\lambda f/2} \cdot \lambda \nabla f/2|^2d\mu$$
				Note that if $f$ is 1-Lipschitz, we have $|\nabla e^{\lambda f/2}|\le |\frac{1}{2}\lambda e^{\lambda f/2}|$, and therefore
				$$\ent_{\mu}[e^{\lambda f}] \le \frac{\lambda^2}{2c} \int  e^{\lambda f}d\mu = \frac{\lambda^2}{2c} \E_{\mu}[e^{\lambda f}]$$ 
				
\end{proof}

The distributions that we are interested has continuous density function on a finite support and 0 elsewhere. Therefore we need to use a non-continuous version of the Corollary above to be rigorous. 

\begin{corollary}\label{cor:non-continuous-concave-concentration}
	Let $S = [a,b]$ be a finite interval in $\mathbb{R}$. Let $d\mu = e^{-U}dx$ for $x \in S$ and $d\mu = 0$ for $x \not\in S$. Suppose for some $c >0$, we have $\nabla^2 U(x)\ge c$ for all $x\in S$.Then the conclusion of Corollary~\ref{cor:log-concave-concentration-for-herbst} is still true. 
	
\end{corollary}

\begin{proof}[Proof of Corollary~\ref{cor:non-continuous-concave-concentration}]
	We first extend Theorem~\ref{thm:log-concave-concentration} to the finite support case. Let $g$ be an extension of $f$ to $\mathbb{R}$, such that $g$ is nonnegative and bounded above by some constant $C$, and $\nabla g$ is also bounded by $C$. Let $U_n$ be a series of extensions of $U$ to $\mathbb{R}$ such that the following happens: a) $U_n$ is twice-differentiable b) $\nabla^2 U_n(x) \ge c$ for all $x\in \mathbb{R}$ c) $\mu_n = e^{-U_n}dx$ approaches to $\mu$ in TV norm as $n$ tends to infinity. (The following choice will work for example, $U_n(x) = U(x) + \mathbf{1}_{x > b} \cdot \left(\nabla U(b)(x-b) + \nabla^2 U(b)(x-b)^2 + \exp(n(x-b)^4)\right) $ $+ \mathbf{1}_{x < a} \cdot \left(\nabla U(b)(x-a) + \nabla^2 U(b)(x-a)^2 + \exp(n(x-a)^4)\right) $. )
	
	Since $g$ and $\nabla g$ are bounded, we have that $|\Exp_{\mu_n}(g^2) - \Exp_{\mu}(g^2)| = \int g^2(d\mu_n-d\mu) \le C^2\TV{\mu_n-\mu} \rightarrow 0$ as $n$ tends to infinity. Similarly we have that $\ent_{\mu_n}(g^2) \rightarrow \ent_{\mu}(g^2)$ and $\Exp_{\mu_n}[|\nabla g|^2] \rightarrow \Exp_{\mu}[|\nabla g|^2]$. Note that under $\mu$, $g$ agrees with $f$ and therefore we have that $\ent_{\mu_n}(g^2) \rightarrow \ent_{\mu}(f^2)$ and $\Exp_{\mu_n}[|\nabla g|^2] \rightarrow \Exp_{\mu}[|\nabla f|^2]$. 
	
	Also note that $\mu_n$ satisfies the condition of Theorem~\ref{thm:log-concave-concentration}, therefore 
		$$\ent_{\mu_n}(g^2)\le \frac{2}{c}\int |\nabla g|^2d\mu_n$$
		
	and the desired result follows by taking $n$ to infinity. 	
\end{proof}

Finally we provide the proof of Lemma~\ref{lem:checking-log-concave}, which is obtained by direct calculation of the second derivatives of $u(x)$. 



\subsection{SDPI for truncated Gaussian}\label{subsec:truncated-gaussian}


We first check that the Lipschitz constants for $f_v(x) = \Pr[V=0\mid X=x] $ as defined in Lemma~\ref{lem:specical_case_Raginsky}. The proof of the following lemma is deferred to Section~\ref{subsec:lipschitz}. 
\begin{lemma}\label{lem:lipschitz}
	When $X$ is generated by $X\sim \mu_v$ conditioned on $V=v$, let $f_v(x) = \Pr[V=0\mid X=x] $, we have that $f_v(x)$ is $\mu/4\sigma^2$-Lipschitz for any $v\in \{0,1\}$. 
\end{lemma}

We first prove Theorem~\ref{thm:truncated-guassian-dpi-info} using Lemma~\ref{lem:checking-log-concave},  Theorem~\ref{thm:log-concave-transportation} and Lemma~\ref{lem:specical_case_Raginsky}. 

\begin{proof}[Proof of Theorem~\ref{thm:truncated-guassian-dpi-info}]
	Note that by definition on support $[-\tau,\tau]$,  $d \mu_0' = \gamma_0\exp(-u_0(x))dx$, and $\d m_0' = \gamma_1\exp(-u_0(x))dx$ with $u_0(x) = -\frac{x^2}{2\sigma^2}$ and $u_1(x) = -\frac{(x-\delta)^2}{2\sigma^2}$. By Lemma~\ref{lem:checking-log-concave}, we have that $\mu = (\mu_0'+\mu_1')/2$ is $1/\sigma^2$-log concave, and therefore by Theorem~\ref{thm:log-concave-transportation}, we have 
		\begin{equation*}
		w_1(\nu,\mu)^2 \le 2\sigma^2\cdot \KL(\nu\|\mu). 
		\end{equation*}		By Lemma~\ref{lem:lipschitz}, we have that $f_v$'s are $\delta/4\sigma^2$-Lipschitz and therefore by Lemma~\ref{lem:specical_case_Raginsky}, we have that \begin{equation*}
		\beta(\mu_0,\mu_1)  \le \delta^2/\sigma^2
		\end{equation*}	
\end{proof}
Then we present the proof of Corollary~\ref{cor:truncated-gaussian-dpi-many-samples}, which relies on the following observation, whose proof is given in Section~\ref{subsec:suff_statistics}.

\begin{lemma}{\label{suff_statistics}}
	Suppose $V\rightarrow (X_1,\dots,X_n)\rightarrow \Pi$ forms a Markov Chain, where conditioned on $V=v$, $(X_1,\dots,X_n)$ are distributed according to $\tilde{\mu}_v$. Then $V\rightarrow X_1+\dots+X_n\rightarrow (X_1,\dots,X_n)\rightarrow \Pi$ also forms a Markov Chain. 
\end{lemma}

Now we are ready to prove Corollary~\ref{cor:truncated-gaussian-dpi-many-samples}. 
\begin{proof}{(Of corollary \ref{cor:truncated-gaussian-dpi-many-samples})}
	Let us restate what we want to prove. Suppose $V \sim B_{1/2}$, $(X_1,\dots,X_n) | V = 0 \sim \tilde{\mu}_0$ and $(X_1,\dots,X_n) | V = 1 \sim \tilde{\mu}_1$ and $V \rightarrow (X_1,\dots,X_n) \rightarrow \Pi$ be a Markov chain. Then 
\begin{align*}
I(\Pi;V) \le \frac{n\delta^2}{\sigma^2} I(\Pi; X_1,\dots,X_n)
\end{align*}
By lemma \ref{suff_statistics}, $V\rightarrow X_1+\dots+X_n\rightarrow (X_1,\dots,X_n)\rightarrow \Pi$ also forms a Markov chain. Then
\begin{align*}
I(\Pi;V) \le \frac{n\delta^2}{\sigma^2} I(\Pi; X_1+\dots+X_n) \le \frac{n\delta^2}{\sigma^2} I(\Pi; X_1,\dots,X_n)
\end{align*}
where the first inequality follows from Theorem~\ref{thm:truncated-guassian-dpi-info} and the fact that the distribution of $X_1+\dots+X_n | V = 0$ is the Gaussian $\mathcal{N}(0,n \sigma^2)$ truncated to $[-\tau, \tau]$ and the distribution of $X_1+\dots+X_n | V = 1$ is the Gaussian $\mathcal{N}(n \delta,n \sigma^2)$ truncated to $[-\tau, \tau]$. The second inequality follows from data processing. 
\end{proof}

\paragraph{Acknowledgments.} 
We thank Yuchen Zhang for suggesting to us the version of of sparse Gaussian mean estimation with signal strength assumption. We are indebted to Ramon van Handel for helping us for proving transportation inequality for truncated Gaussian distribution. 
Mark Braverman would like to thank the support in part by an NSF CAREER award (CCF-1149888), NSF CCF-1525342, a Packard Fellowship in Science and Engineering, and the Simons Collaboration on Algorithms and Geometry. Ankit Garg would like to thank the support by a Simons Award in Theoretical Computer Science and a Siebel Scholarship. Tengy Ma would like to thank the support by a Simons Award in Theoretical Computer Science and IBM PhD Fellowship. D. Woodruff would like to thank the support from XDATA program of the Defense Advanced Research Projects Agency  (DARPA), administered through Air Force Research Laboratory FA8750-12-C-0323. 

\bibliography{ref}

\begin{thebibliography}{BYJKS04}

\bibitem[AG76]{Ahlswede-Gacs}
R.~Ahlswede and P.~Gacs.
\newblock Spreading of sets in product spaces and hypercontraction of the
  markov operator.
\newblock {\em Annals of Probability}, 4:925--939, 1976.

\bibitem[BG99]{bobkov1999exponential}
Sergej~G Bobkov and Friedrich G{\"o}tze.
\newblock Exponential integrability and transportation cost related to
  logarithmic sobolev inequalities.
\newblock {\em Journal of Functional Analysis}, 163(1):1--28, 1999.

\bibitem[BJKS04]{DBLP:journals/jcss/Bar-YossefJKS04}
Ziv Bar{-}Yossef, T.~S. Jayram, Ravi Kumar, and D.~Sivakumar.
\newblock An information statistics approach to data stream and communication
  complexity.
\newblock {\em J. Comput. Syst. Sci.}, 68(4):702--732, 2004.

\bibitem[BYJKS04]{BJKS04}
Ziv Bar-Yossef, T.~S. Jayram, Ravi Kumar, and D.~Sivakumar.
\newblock An information statistics approach to data stream and communication
  complexity.
\newblock {\em J. Comput. Syst. Sci.}, 68(4), 2004.

\bibitem[CR11]{CR11}
Amit Chakrabarti and Oded Regev.
\newblock An optimal lower bound on the communication complexity of
  gap-hamming-distance.
\newblock {\em STOC}, 2011.

\bibitem[DAW12]{Duchi12}
John~C Duchi, Alekh Agarwal, and Martin~J Wainwright.
\newblock Dual averaging for distributed optimization: convergence analysis and
  network scaling.
\newblock {\em Automatic Control, IEEE Transactions on}, 57(3):592--606, 2012.

\bibitem[DJWZ14]{DBLP:journals/corr/DuchiJWZ14}
John~C. Duchi, Michael~I. Jordan, Martin~J. Wainwright, and Yuchen Zhang.
\newblock Information-theoretic lower bounds for distributed statistical
  estimation with communication constraints.
\newblock {\em CoRR}, abs/1405.0782, 2014.

\bibitem[GL10]{gozlan2010transport}
Nathael Gozlan and Christian L{\'e}onard.
\newblock Transport inequalities. a survey.
\newblock {\em arXiv preprint arXiv:1003.3852}, 2010.

\bibitem[GMN14]{DBLP:conf/nips/GargMN14}
Ankit Garg, Tengyu Ma, and Huy~L. Nguyen.
\newblock On communication cost of distributed statistical estimation and
  dimensionality.
\newblock In Zoubin Ghahramani, Max Welling, Corinna Cortes, Neil~D. Lawrence,
  and Kilian~Q. Weinberger, editors, {\em Advances in Neural Information
  Processing Systems 27: Annual Conference on Neural Information Processing
  Systems 2014, December 8-13 2014, Montreal, Quebec, Canada}, pages
  2726--2734, 2014.

\bibitem[Jay09]{Hellinger}
T.S. Jayram.
\newblock Hellinger strikes back: A note on the multi-party information
  complexity of and.
\newblock In Irit Dinur, Klaus Jansen, Joseph Naor, and Jos{\'e} Rolim,
  editors, {\em Approximation, Randomization, and Combinatorial Optimization.
  Algorithms and Techniques}, volume 5687 of {\em Lecture Notes in Computer
  Science}, pages 562--573. Springer Berlin Heidelberg, 2009.

\bibitem[KVW14]{kvw14}
Ravi Kannan, Santosh Vempala, and David~P. Woodruff.
\newblock Principal component analysis and higher correlations for distributed
  data.
\newblock In {\em Proceedings of The 27th Conference on Learning Theory, {COLT}
  2014, Barcelona, Spain, June 13-15, 2014}, pages 1040--1057, 2014.

\bibitem[LBKW14]{lbkw14}
Yingyu Liang, Maria{-}Florina Balcan, Vandana Kanchanapally, and David~P.
  Woodruff.
\newblock Improved distributed principal component analysis.
\newblock In {\em Advances in Neural Information Processing Systems 27: Annual
  Conference on Neural Information Processing Systems 2014, December 8-13 2014,
  Montreal, Quebec, Canada}, pages 3113--3121, 2014.

\bibitem[Led01]{Ledoux01}
Michel Ledoux.
\newblock {\em The Concentration of Measure Phenomenon}, volume~89 of {\em
  Mathematical Surveys and Monographs}.
\newblock American Mathematical Society, 2001.

\bibitem[LSLT15]{LSQT15}
Jason~D Lee, Yuekai Sun, Qiang Liu, and Jonathan~E Taylor.
\newblock Communication-efficient sparse regression: a one-shot approach.
\newblock {\em arXiv preprint arXiv:1503.04337}, 2015.

\bibitem[Rag14]{Raginsky14}
Maxim Raginsky.
\newblock Strong data processing inequalities and
  {\textdollar}{\(\Phi\)}{\textdollar}-sobolev inequalities for discrete
  channels.
\newblock {\em CoRR}, abs/1411.3575, 2014.

\bibitem[SD15]{duchi15}
Jacob Steinhardt and John~C. Duchi.
\newblock Minimax rates for memory-bounded sparse linear regression.
\newblock In {\em Proceedings of The 28th Conference on Learning Theory, {COLT}
  2015, Paris, France, July 3-6, 2015}, pages 1564--1587, 2015.

\bibitem[Sha14]{Shamir14online}
Ohad Shamir.
\newblock Fundamental limits of online and distributed algorithms for
  statistical learning and estimation.
\newblock In Z.~Ghahramani, M.~Welling, C.~Cortes, N.D. Lawrence, and K.Q.
  Weinberger, editors, {\em Advances in Neural Information Processing Systems
  27}, pages 163--171. Curran Associates, Inc., 2014.

\bibitem[SSZ14]{SSZ14}
Ohad Shamir, Nathan Srebro, and Tong Zhang.
\newblock Communication-efficient distributed optimization using an approximate
  newton-type method.
\newblock In {\em Proceedings of the 31th International Conference on Machine
  Learning, {ICML} 2014, Beijing, China, 21-26 June 2014}, pages 1000--1008,
  2014.

\bibitem[WZ12]{WZ12}
David~P. Woodruff and Qin Zhang.
\newblock Tight bounds for distributed functional monitoring.
\newblock {\em STOC}, 2012.

\bibitem[ZDJW13]{ZDJW13}
Yuchen Zhang, John~C. Duchi, Michael~I. Jordan, and Martin~J. Wainwright.
\newblock Information-theoretic lower bounds for distributed statistical
  estimation with communication constraints.
\newblock In {\em NIPS}, pages 2328--2336, 2013.

\bibitem[ZDW13]{DBLP:journals/jmlr/ZhangDW13}
Yuchen Zhang, John~C. Duchi, and Martin~J. Wainwright.
\newblock Communication-efficient algorithms for statistical optimization.
\newblock {\em Journal of Machine Learning Research}, 14(1):3321--3363, 2013.

\bibitem[ZX15]{ZhangXiao}
Yuchen Zhang and Lin Xiao.
\newblock Communication-efficient distributed optimization of self-concordant
  empirical loss.
\newblock {\em CoRR}, abs/1501.00263, 2015.

\end{thebibliography}
\bibliographystyle{alpha}
\appendix

\section{Proofs of Results in Section~\ref{sec:gaussian-mean}}\label{sec:app:thm:multi-dpi-gaussian}

In this section, we prove Theorem~\ref{thm:multi-dpi-gaussian} and Corollary~\ref{cor:sparse-linear-regression}.

\begin{proof}[Proof of Corollary~\ref{cor:sparse-linear-regression}]
	Suppose there exists such a protocol with mean-squared loss $R$ and communication cost $C$ for sparse linear regression problem $\SLR(n,m,k,d,\sigma^2)$. We are going to use it to solve the sparse linear regression problem $\SGME(m,1,d,k,\sigma_0)$ as follows. Suppose the $i^{th}$ machine has data $X_i \sim \mathcal{N}(\theta, \sigma_0^2 I_{d\times d})$ with $\sigma_0 = \frac{\sigma}{\lambda\sqrt{n}}$. Then the machines can prepare 
	$$y_{S_i} = A_{S_i}X_i + b_{i}$$ where 
	$b_i \sim \mathcal{N}(0, \sigma^2 I - \sigma_0^2A_{S_i}A_{S_i}^T)$. Note that by the bound $\|A_{S_i}\|\le \lambda/\sqrt{n}$, we have that $\sigma^2 I - \sigma_0^2A_{S_i}A_{S_i}^T$ is positive semidefinite. Note that then $y_{S_i}$ can written in the form 
	
	$$y_{S_i} = A_{S_i} \theta + \xi_i$$
	where $\xi_i$'s are independent distributed according to $\mathcal{N}(0,\sigma^2I_{n\times n})$
	
	Then the machines call the protocol for the sparse linear regression problem with data $(y_{S_i},A_{S_i})$. Therefore we obtain a protocol that solves $\SGME(m,1,d,k,\sigma_0)$ with communication $R$ and $C$. 
	Then by Theorem~\ref{thm:sparse-gaussian-mean}, we know that $$R\cdot C \ge \Omega(\sigma_0^2kd) = \Omega(\frac{\sigma^2kd}{\lambda^2 n})$$
\end{proof}


\section{Tight Upper Bound with One-way Communication}\label{sec:one-way}


In this section, we describe a one-way communication protocol
achieving the tight minimal communication for Gaussian mean estimation problem $\GME(n,m,d,\sigma^2)$ with the assumption that $|\theta|_{\infty}\le \frac{\sigma}{\sqrt{n}}$. 


Note that for the design of protocol, it suffices to consider a one-dimensional problem.  Protocol~\ref{pro:one-way-alg} solves the one-dimensional Gaussian mean estimation problem, with 
each machine sending exactly $1$ bit, and therefore the total communication is $m$ bits. To get a $d$-dimensional protocol, we just need to apply Protocol~\ref{pro:one-way-alg} to each dimension.  In order to obtain the tradeoff as stated in Theorem~\ref{thm:one-way-upper-bound}, one needs to  run Protocol~\ref{pro:one-way-alg} on the first $\alpha m$ machines, and let the other machines be idle. 
%

\begin{Protocol}
Unknown parameter $\theta \in [-\sigma/\sqrt{n},\sigma/\sqrt{n}]$\\
Inputs: Machine $i$ gets $n$ samples $(X_i^{(1)}, \ldots, X_i^{(n)})$ where $X_i^{(j)} \sim \mathcal{N}(\theta, \sigma)$.
\begin{itemize}
\item Simultaneously, each machine $i$
\begin{enumerate}
\item Computes $X_i =
\frac{1}{\sigma \sqrt{n}}\sum_{j=1}^n X_i^{(j)}$
\item Sends $B_i$ 
\[
B_i = \left\{
\begin{array}{ll}
1 &\mbox{ if } X_i \ge 0\\
-1 &\mbox{ otherwise } 
\end{array}\right.
\]
\end{enumerate}
\item Machine $1$ computes
\[ T = \sqrt{2} \cdot \erf^{-1} \left (\frac{1}{m} \sum_{i=1}^m B_i
\right )\] where $\erf^{-1}$ is the inverse of the Gauss error function.
\item It returns the estimate $\hat{\theta} =
  \frac{\sigma}{\sqrt{n}}\hat{\theta}'$ where $\hat{\theta}' = \max(\min(T, 1), -1)$ is obtained by truncating $T$ to the interval $[-1, 1]$. 
\end{itemize}
\caption{\label{pro:one-way-alg}A simultaneous algorithm for estimating the mean of a normal
  distribution in the distributed setting.}
\end{Protocol}

The correctness of the protocol
follows from the following theorem. 

\begin{theorem}
The algorithm described in Protocol~\ref{pro:one-way-alg} uses $m$ bits of communication and achieves the
following mean squared loss.
\[
\E\left[(\hat{\theta} - \theta)^2\right] = O\left(\frac{\sigma^2}{mn}\right)
\]
where the expectation is over the random samples and the random coin
tosses of the machines. 
\end{theorem}

\begin{proof}
Let $\bar{\theta} = \theta \sqrt{n}/\sigma$. 

Notice that $X_i$ is distributed
according to $\mathcal{N}(\bar{\theta}, 1)$. 
Our goal is to estimate $\bar{\theta}$ from the $X_i$'s. 
By our assumption on $\theta$, we have $\bar{\theta} \in [-1,1]$. 


The random variables $B_i$ are independent with each other. 
We consider the mean and variance of $B_i$'s. For the mean we have that,

\begin{eqnarray*}
\E \left[B_i\right] & = &
\E \left[2 \cdot \Pr [0 \leq X_i] - 1\right]\\
\end{eqnarray*}

For any $i \in [m]$, 
$\Pr[0 \leq X_i] = \Pr[-X_i \leq 0] = \Phi_{-\bar{\theta},1}(0)$, where $\Phi_{\mu,\sigma^2}$ is the CDF of normal distribution $\mathcal{N}(\mu,\sigma^2)$. Note the following relation
between the error function and the CDF of a normal random variable
\[
\Phi_{\mu,\sigma^2}(x) = \frac{1}{2} + \frac{1}{2}\erf\left(\frac{x-\mu}{\sqrt{2\sigma^2}}\right)
\]
Hence, 
$$\E \left[B_i\right] =  \textrm{erf}(\bar{\theta}/\sqrt{2}).$$
Let $B = \frac{1}{m}\sum_{i=1}^m B_i$, then we have that $\Exp[B] = \erf(\bar{\theta}/\sqrt{2})\le \erf(1/\sqrt{2})$ and therefore by a Chernoff bound, 
the probability that
$B > \erf(1)$ or $B \le \erf(-1)$ 
is $\exp(-\Omega(m))$. Thus, with probability at least $1-\exp(-\Omega(m))$,
we have $\erf(-1)\le B\le \erf(1)$ and therefore $|T| \le \sqrt{2}$. 

Let $\mathcal{E}$ be the event that $|T|\le \sqrt{2}$, then we have that the error of $\bar{\theta}$ is bounded by 

\begin{align*}
\E[|\hat{\theta}' - \bar{\theta}|^2] &= \E[|\hat{\theta}' - \bar{\theta}|^2 \mid 
                                  \mathcal{E}] \Pr[ \mathcal{E}] + \E[
                                  |\hat{\theta}'
                                   - \bar{\theta}|^2 \mid  \bar{\mathcal{E}}
                                ] \Pr[ \bar{\mathcal{E}}]\\
 &\le \E[|\sqrt{2}\erf^{-1}(B) - \sqrt{2}\erf^{-1}(\Exp[B])|^2 \mid 
                                                                  \mathcal{E}]  \Pr[ \mathcal{E}]+ 2\Pr[ \bar{\mathcal{E}}] \\
&= \E[|\sqrt{2}\erf^{-1}(B) - \sqrt{2}\erf^{-1}(\Exp[B])|^2 \mid 
                                                                 \mathcal{E}]  \Pr[ \mathcal{E}]+ 2\exp(-\Omega(m))
\end{align*}
Let $M = \max_{\textrm{erf}^{-1}(x) \in [-1, 1]} \frac{d \textrm{erf}^{-1}(x)}{dx} < 3$. Then we have that $|\erf^{-1}(x) - \erf^{-1}(y)|\le M|x-y|\le O(1)\cdot |x-y|$ for any $x,y\in [-1,1]$. Therefore it follows that 
\begin{align*}
	\E[|\hat{\theta}' - \bar{\theta}|^2] 
	&\le \E[|\sqrt{2}\erf^{-1}(B) - \sqrt{2}\erf^{-1}(\Exp[B])|^2 \mid 
	\mathcal{E}]  \Pr[ \mathcal{E}]+ 2\exp(-\Omega(m))\\
	&\le \Exp[2M^2|B-\Exp[B]|^2 \mid \mathcal{E}] \Pr[ \mathcal{E}] +2\exp(-\Omega(m))\\
		&\le \Exp[2M^2|B-\Exp[B]|^2]+2\exp(-\Omega(m))\\
		& \le O\left(\frac{1}{m}\right) + 2\exp(-\Omega(m))\\
		&\le O\left(\frac{1}{m}\right)
	\end{align*}
	
Hence we have that $$\Exp\left[|\hat{\theta} - \theta|^2 \right]  = \frac{\sigma^2}{n}\Exp\left[|\hat{\theta}' - \bar{\theta}|^2 \right] = O\left(\frac{\sigma^2}{mn}\right)$$
\end{proof}
%
\subsection{Extension to general $\theta$}\label{sec:generalUpper}
Now we do not assume that $\theta_{\ell} \in [-\sigma/\sqrt{n},\sigma/\sqrt{n}]$ for each dimension
$\ell \in [d]$, and still show how to
achieve a $1$-round protocol with $O(md)$ bits of communication, up to low order terms. 
We will make the simplifying and standard assumptions though, 
that $|\theta_{\ell}| \leq U = \poly(md)$ for each $\ell \in [d]$, as well
as $\log(mdn/\sigma) = o(m)$ and $mdn/\sigma \geq (mdn)^c$ for a constant $c > 0$. 

\paragraph{The protocol.}

\begin{Protocol}
Unknown parameter $\theta$\\
Inputs: Machine $i$ gets $n$ samples $(X_i^{(1)}, \ldots, X_i^{(n)})$ where $X_i^{(j)} \sim \mathcal{N}(\theta, \sigma)$.
\begin{itemize}
\item Simultaneously, each machine $i$
\begin{enumerate}
\item Computes $X_i = \frac{1}{\sigma \sqrt{n}}\sum_{j=1}^n X_i^{(j)}$
\item If $i \leq r = O(\log (mdn/\sigma))$, machine $i$ sends its first $O(\log(mdn/\sigma))$ bits of $X_i$ to the coordinator (Machine $1$)
\item Else if $i > r$, machine $i$
\begin{enumerate}
\item Computes $R_i = X_i - \lfloor X_i \rfloor$, $R_i' = X_i + 1/5 - \lfloor X_i + 1/5 \rfloor$
\item Sends $B_i$ and $B_i'$
\[
B_i = \left\{
\begin{array}{ll}
1 &\mbox{ with probability } R_i\\
0 &\mbox{ with probability } 1-R_i
\end{array}\right.
\]
\[
B_i' = \left\{
\begin{array}{ll}
1 &\mbox{ with probability } R_i'\\
0 &\mbox{ with probability } 1-R_i' 
\end{array}\right.
\]
\end{enumerate}
\end{enumerate}
\item Machine $1$
\begin{enumerate}
\item Computes an estimate $\gamma = \frac{\sqrt{n}}{\sigma}$ times the median of $X_i$'s sent by the first $r$ machines.
\item Computes \[ T = \frac{1}{m-r} \sum_{i=r+1}^m B_i, T' = \frac{1}{m-r} \sum_{i=r+1}^m B_i'\] 
\item Returns $\frac{\sigma}{\sqrt{n}}\hat{\theta}$ where $\hat{\theta}$ is a multiple of 
$1/\sqrt{m-r}$ satisfying $|\gamma - \hat{\theta}| < 1/100$ and certain agreement conditions with $T, T'$ described in the text.
\end{enumerate}
\end{itemize}
\caption{\label{pro:one-way-alg2}A simultaneous algorithm for estimating the mean of a normal
  distribution in the distributed setting without assuming $|\theta| \le \sigma/\sqrt{n}$.}
\end{Protocol}
%
%
As before, it suffices to consider a one-dimensional problem. 
 Protocol~\ref{pro:one-way-alg2} solves the one-dimensional Gaussian mean estimation
problem using 
$O(m + \log^2(mdn/\sigma))$ bits of communication. To solve the $d$-dimensional problem,
we run the protocol independently on each coordinate. The total
communication will be $O(md + d \log^2(mdn/\sigma))$ bits. 
We fix $\ell \in [d]$ and let $\theta = \theta_{\ell}$. 
Let $\bar{\theta} = \theta \sqrt{n}/\sigma$, where now we no longer assume 
$\bar{\theta} \leq 1$.
We will show the output
$\hat{\theta}$ satisfies:
$$\E[|\hat{\theta} - \bar{\theta}|^2] = O \left (\frac{1}{m} \right ),$$
from which it follows that
$$\E[|\frac{\sigma}{\sqrt{n}} \hat{\theta} - \theta|^2] = O \left (\frac{\sigma^2}{mn} \right ).$$
We now describe the one-dimensional problem for a given unknown mean $\bar{\theta}$. 
The first $r = O(\log (mdn/\sigma))$
machines $i$ send the first $O(\log(mdn/\sigma))$ bits of their (averaged) 
input Gaussians $X_i = \frac{1}{\sigma \sqrt{n}} \sum_{j=1}^n X_i^{(j)}$ 
to the coordinator.
Note that the random variables $X_i$
are distributed according to $\mathcal{N}(\bar{\theta}, 1)$. 

Since $O(\log(mdn/\sigma))$ bits of each $X_i$ are communicated to the coordinator, 
since $\bar{\theta} \leq \poly(md) \cdot \sqrt{n}/\sigma$ 
(here we use our assumption that $|\theta_{\ell}| \leq \poly(md)$ for each $\ell \in [d]$), 
and since each $X_i$ has variance $1$, it follows by standard Chernoff bounds 
that the median $\gamma$ of $X_1, \ldots, X_r$ is within an additive $\frac{1}{100}$ of 
$\bar{\theta}$ with probability $1-\frac{1}{(mdn/\sigma)^{\alpha}}$ for an arbitrarily
large constant $\alpha > 0$ depending on the value $r = O(\log(mdn/\sigma))$.  
We call this event $\mathcal{E}$, so $\Pr[\mathcal{E}] \geq 1-\frac{1}{(mdn/\sigma)^{\alpha}}$. 

In parallel, machines $r+1, r+2, \ldots, m$ do the following. Let $R_i \in [0,1)$
be such that 
$R_i = X_i - \lfloor X_i \rfloor$. Similarly, let $R_i' \in [0,1)$ 
be such that $R_i' = X_i + 1/5 - \lfloor X_i + 1/5 \rfloor$. 

For $i = r+1, \ldots, m$, the $i$-th machine sends a bit $B_i \in \{0,1\}$,
where 
$$\Pr[B_i = 1] = R_i,$$
and the $i$-th matchine also sends a bit $B_i' \in \{0,1\}$
where 
$$\Pr[B_i' = 1] = R_i'.$$
We describe the output of the coordinator in the proof of correctness below. 
Observe that the overall communication is 
$O(m + \log^2(mdn/\sigma))$, as desired. 

\paragraph{Correctness.} Consider the ``sawtooth'' wave $f(x)$, which for a parameter $L$,
satisfies $f(x) = x/(2L)$ for $x \in [0, 2L)$, and is periodic with period $2L$. Its Fourier
series\footnote{See, e.g., \url{http://mathworld.wolfram.com/FourierSeriesSawtoothWave.html}} is given by
$$f(x) = \frac{1}{2} - \frac{1}{\pi} \sum_{k=1}^{\infty} \frac{1}{k} \sin \left (\frac{k \pi x}{L} \right ).$$
We set $L = 1/2$ and note that $f(X_i) = R_i$. 
Then, for $X \sim N(\bar{\theta}, 1)$, using a standard transformation of the Gaussian distribution, 
$${\bf E}[\sin(tX)] = e^{-t^2/2} \sin(t \bar{\theta}),$$
we have 
\begin{eqnarray*}
{\bf E}[B_i] & = & {\bf E}[R_i]\\
& = & {\bf E}[f(X_i)]\\
& = & \frac{1}{2} - \frac{1}{\pi} \sum_{k=1}^{\infty} \frac{1}{k} e^{-(k \pi/L)^2/2} \sin(k \pi \bar{\theta}/L)\\
& = & \frac{1}{2} - \frac{1}{\pi} \sum_{k=1}^{\infty} \frac{1}{k} e^{-2 k^2 \pi^2} \sin(2 k \pi \bar{\theta}).
\end{eqnarray*}
%
%
Let $B = \frac{1}{m} \sum_{i=r+1}^m B_i$, so that
${\bf E}[B] = {\bf E}[B_i]$. Since the $B_i$ are Bernoulli random variables,
\begin{eqnarray}\label{eqn:concentrate}
{\bf E}[|B - {\bf E}[B]|^2] \leq \frac{1}{m-r} \leq \frac{2}{m},
\end{eqnarray}
where the second inequality uses that $r = O(\log(mdn/\sigma))$ is at most $m/2$ under our assumption that
$\log(mdn/\sigma) = o(m)$. 
%
In an analogous fashion the coordinator computes a $B'$ using the $B'_i$. 

If event $\mathcal{E}$ occurs, then 
the coordinator knows $\gamma$ satisfying 
$|\gamma - \bar{\theta}| < \frac{1}{100}$, and
using $\gamma$ together with $B$, 
will output its estimate to $\bar{\theta}$ as follows.
Let $\{x\} = x - \lfloor x \rfloor$. 
The coordinator checks which of the two conditions $\gamma$ satisfies:
\begin{enumerate}
\item $1/50 < \{\gamma\} < 49/50$ and $|\{\gamma\} - 1/4| \geq 3/100$ and $|\{\gamma\} - 3/4| \geq 3/100$
\item $1/50 < \{\gamma + 1/5\} < 49/50$ and $|\{\gamma + 1/5\} - 1/4| \geq 3/100$ and $|\{\gamma + 1/5\} - 3/4| \geq 3/100$. 
\end{enumerate}
We note that one of these two conditions must be satisfied. To see this, suppose the first condition is not
satisfied. If it is not satisfied because $\{\gamma\} < 1/50$, then $\{\gamma + 1/5\} \in [1/5, 1/5+1/50]$, which satisfies
the second of the two conditions. If it is not satisfied because $\{\gamma\} > 49/50$, then $\{\gamma + 1/5\} \in [1/5-1/50, 1/5]$,
which satisfies the second of the two conditions. If the first condition is not satisfied because
$\{\gamma\} \in [1/4 - 1/50, 1/4+1/50]$, then $\{\gamma + 1/5\} \in [9/20 -1/50, 9/20 + 1/50]$ and the second condition is satisfied.
If the first condition is not satisfied because $\{\gamma\} \in [3/4-1/50, 3/4 + 1/50]$, then
$\{\gamma + 1/5\} \in [19/20 - 1/50, 19/20 + 1/50]$, which satisfies the second condition. 

If the first condition holds, the coordinator will use $B$ and estimate $\bar{\theta}$ below, otherwise it will use $B'$
and estimate $\bar{\theta} + 1/5$ below. We will analyze the first case; the second case is analogous. Note
that since $\{\gamma\} > 1/50$, and $|\gamma - \bar{\theta}| < \frac{1}{100}$, the coordinator 
learns $Z = \lfloor \bar{\theta} \rfloor$. Its estimate $\hat{\theta}$ for
$\bar{\theta}$ is then $Z + g(B)$, for a function $g(B)$ to be specified 
(in the other case the coordinator would have learned
$\{\bar{\theta} + 1/5\}$ and $\hat{\theta}$ would have been $\{\bar{\theta} + 1/5\} + g(B') - 1/5$). 

To define $g(B)$, we need the following claim. Note that in the first case $|\{\gamma\}-1/4| \geq 3/100$ and
so by the triangle inequality $|\{\bar{\theta}\}-1/4| \geq 3/100 - \gamma = 1/50$. Similarly,
$|\{\bar{\theta}\}-3/4| \geq 1/50$, so the conditions of the following claim hold for $\{\bar{\theta}\}$. 

\begin{claim}\label{claim:relate}
Define 
$h(x) = \sum_{k=1}^{\infty} \frac{1}{k} e^{-2 k^2 \pi^2} \sin(2 k \pi x)$. 
There exists a constant $C > 0$ with the following guarantee. 
If $|\{\bar{\theta}\} - 1/4| \geq 1/50$ and $|\{\bar{\theta}\} - 3/4| \geq 1/50$ 
then for any number $x \in [\{\bar{\theta}\} - 1/100, \{\bar{\theta}\} + 1/100]$, 
$$C \leq h'(x) \leq 1.$$
\end{claim}
Before proving the claim, we conclude the correctness proof. The coordinator guesses  
$\frac{i}{\sqrt{m}}$ for each integer $i$ for which 
$|Z + \frac{i}{\sqrt{m}} - \gamma| < \frac{1}{100}$. For each guess $\frac{i}{\sqrt{m}}$, the coordinator
checks if 
\begin{eqnarray}\label{eqn:check}
|\sum_{k=1}^{\infty} \frac{1}{k} e^{-2 k^2 \pi^2} \sin(2 k \pi \frac{i}{\sqrt{m}}) - \pi(\frac{1}{2}-B)| \leq \frac{1}{\sqrt{m}}
\end{eqnarray}
Note that, since the above Fourier series is periodic between succesive integers, we need not add $Z$ to $\frac{i}{\sqrt{m}}$
in (\ref{eqn:check}). Let $g(B)$ be the first guess which passes the check. 
The coordinator outputs $\hat{\theta} = Z + g(B)$ as its estimate to $\bar{\theta}$
(the second case is analogous, in which $Z$ corresponds to
$\lfloor \bar{\theta} + 1/5 \rfloor$ and $g(B')$ is defined in the same way). If there is no such $g(B)$ the coordinator
just outputs $\gamma$. Note also that if its output ever exceeds
our assumed upper bound $U = \poly(mnd/\sigma)$ on the magnitude of $\bar{\theta}$, then we instead output $U$. 

Then 
\begin{eqnarray}
{\bf E}[|\hat{\theta} - \bar{\theta}|^2] & = & 
{\bf E}[|\hat{\theta} - \bar{\theta}|^2 \mid \mathcal{E}] \Pr[\mathcal{E}]
+ {\bf E}[|\hat{\theta} - \bar{\theta}|^2 \mid \neg \mathcal{E}] \Pr[\neg \mathcal{E}] \notag \\
& = & {\bf E}[|\hat{\theta} - \bar{\theta}|^2 \mid \mathcal{E}] (1-\frac{1}{(mdn/\sigma)^{\alpha}} )
+ 4U^2 \cdot \frac{1}{(nmd/\sigma)^{\alpha}} \notag \\
& \leq & {\bf E}[|\hat{\theta} - \bar{\theta}|^2 \mid \mathcal{E}] (1-\frac{1}{(mdn)^{c\alpha}} )
+ 4U^2 \cdot \frac{1}{(mdn)^{c\alpha}} \notag \\
& \leq & {\bf E}[|\hat{\theta} - \bar{\theta}|^2 \mid \mathcal{E}] + \frac{1}{m}, \label{eqn:last}
\end{eqnarray}
where the first inequality uses our assumption that $(mdn/\sigma) \geq (mdn)^c$ for a constant $c > 0$, and the
second inequality holds for a sufficiently large constant $\alpha > 0$. 

Conditioned on $\mathcal{E}$, we have $\hat{\theta} - \bar{\theta} = g(B) - \{\theta\}$. 
If (\ref{eqn:check}) holds for a given $\frac{i}{\sqrt{m}}$, then
$$|\sum_{k=1}^{\infty} \frac{1}{k} e^{-2 k^2 \pi^2} \sin(2 k \pi \frac{i}{\sqrt{m}}) - \pi(\frac{1}{2}-B)| \leq \frac{1}{\sqrt{m}}.$$
Let $\mathcal{F}$ be the event that the coordinator finds such an $\frac{i}{\sqrt{m}}$ for which (\ref{eqn:check}) holds. 
We use the shorthand $h(z)$ to denote $\sum_{k=1}^{\infty} \frac{1}{k} e^{-2 k^2 \pi^2} \sin(2 k \pi z)$. 
\begin{eqnarray*}
{\bf E}[|\hat{\theta} - \bar{\theta}|^2 \mid \mathcal{E} \wedge \mathcal{F}]
& = & {\bf E}[|\frac{i}{\sqrt{m}} - \{\bar{\theta}\}|^2 \mid \mathcal{E} \wedge \mathcal{F}]\\
& \leq & {\bf E}[|h(\frac{i}{\sqrt{m}}) - h(\{\bar{\theta}\})|^2 \mid \mathcal{E} \wedge \mathcal{F}]\\
& \leq & {\bf E}[(|h(\frac{i}{\sqrt{m}}) - \pi (\frac{1}{2}-B )| + |\pi (\frac{1}{2}-B ) - h(\{\bar{\theta}\})|)^2 \mid \mathcal{E} \wedge \mathcal{F}]\\ 
& \leq & {\bf E}[(\frac{1}{\sqrt{m}} + |\pi (\frac{1}{2}-B ) - \pi (\frac{1}{2} - {\bf E}[B] )|)^2 \mid \mathcal{E} \wedge \mathcal{F}]\\
& \leq & {\bf E}[(\frac{1}{\sqrt{m}} + \pi |B - {\bf E}[B]|)^2 \mid \mathcal{E} \wedge \mathcal{F}]\\
& \leq & \frac{2}{m} + 2\pi^2 {\bf E}[|B-{\bf E}[B]|^2 \mid \mathcal{E} \wedge \mathcal{F}]
\end{eqnarray*}
where the first equality follows from $\hat{\theta} - \bar{\theta} = g(B) - \{\theta\}$, the first inequality uses 
the fact that the algorithm ensures $|\frac{i}{\sqrt{m}} - \{\bar{\theta}\}| \leq \frac{1}{100}$ given that $\mathcal{E}$ occurs and
therefore one can apply Claim \ref{claim:relate} with $x = \frac{i}{\sqrt{m}}$ to conclude that 
$|h(\frac{i}{\sqrt{m}}) - h(\{\bar{\theta}\})| \leq |\frac{i}{\sqrt{m}} - \{\bar{\theta}\}|,$  
the second inequality is the triangle inequality, the third inequality uses the guarantee on the value $\frac{i}{\sqrt{m}}$ 
chosen by the coordinator
and the definition of ${\bf E}[B]$, the fourth inequality rearranges terms, and the fifth inequality uses $(a+b)^2 \leq 2a^2 + 2b^2$. 

If there is no value $\frac{i}{\sqrt{m}}$ for which (\ref{eqn:check}) holds, then since $\mathcal{E}$ occurs
it means there is no integer multiple of $\frac{1}{\sqrt{m}}$, call it $x$, 
with $|x -\{\bar{\theta}\}| \leq \frac{1}{100}$ for which 
$|h(x) - \pi(\frac{1}{2}-B)| \leq \frac{1}{\sqrt{m}}$.
If it were the case 
that $|{\bf E}[B] - B| < \frac{C}{100\pi}$, where $C > 0$ is the constant of Claim \ref{claim:relate},
then $|\frac{1}{2} - \frac{1}{\pi} h(\bar{\theta})-B| < \frac{C}{100 \pi}$, or
equivalently, $|\pi(\frac{1}{2} - B) - h(\bar{\theta})| < \frac{C}{100}$. By Claim \ref{claim:relate}, though,
we can find an $x$ which is an integer multiple of $\frac{1}{\sqrt{m}}$ which is within $\frac{1}{\sqrt{m}}$ of $y$,
where $h(y) = \pi(\frac{1}{2} - B)$. This follows since the derivative on $[\{\bar{\theta}\}-1/100, \{\bar{\theta}\}+1/100]$
is at least $C$. But then $|h(x) - h(y)| \leq |x-y| \leq \frac{1}{\sqrt{m}}$, contradicting that (\ref{eqn:check}) did not 
hold. It follows that in this case $|{\bf E}[B] - B| \geq \frac{C}{100\pi}$. Now in this case, 
we obtain an additive $\frac{1}{100}$ 
approximation, and so $|\hat{\theta} - \bar{\theta}|^2 \leq \frac{\pi^2}{C^2}|B - {\bf E}[B]|^2$. Hence, 
$${\bf E}[|\hat{\theta} - \bar{\theta}|^2 \mid \mathcal{E} \wedge \neg \mathcal{F}] \leq O(1) \cdot {\bf E}[|B - {\bf E}[B]|^2 
\mid \mathcal{E} \wedge \neg \mathcal{F}],$$
and so 
\begin{eqnarray*}
{\bf E}[|\hat{\theta} - \bar{\theta}|^2 \mid \mathcal{E}] 
& \leq & {\bf E}[|\hat{\theta} - \bar{\theta}|^2 \mid \mathcal{E}, \mathcal{F}]\Pr[\mathcal{F}]
+ {\bf E}[|\hat{\theta} - \bar{\theta}|^2 \mid \mathcal{E}, \neg \mathcal{F}]\Pr[\neg \mathcal{F}]\\
& \leq & \frac{2}{m} + 2 \pi^2   {\bf E}[|B-{\bf E}[B]|^2 \mid \mathcal{E} \wedge \mathcal{F}]\Pr[\mathcal{F}]
+ O(1) \cdot {\bf E}[|B - {\bf E}[B]|^2 \mid \mathcal{E} \wedge \neg \mathcal{F}] \Pr[\neg \mathcal{F}]\\
& \leq & O\left(\frac{1}{m} \right ) + O(1) \cdot {\bf E}[|B - {\bf E}[B]|^2 \mid \mathcal{E}]\\
& \leq & O \left (\frac{1}{m} \right ),
\end{eqnarray*}
where the final inequality uses ${\bf E}[|B - {\bf E}[B]|^2 \mid \mathcal{E}]
\leq \frac{{\bf E}[|B - {\bf E}[B]|^2]}{\Pr[\mathcal{E}]} \leq 2{\bf E}[|B - {\bf E}[B]|^2]$, and (\ref{eqn:concentrate}). 

Combining this with (\ref{eqn:last}) completes the proof that ${\bf E}[|\hat{\theta} - \bar{\theta}|^2] = O(1/m)$. 
%
%
%
\begin{proof}[Proof of Claim] We need to understand the derivative, with respect to $x$, of the function
$$h(x) = \sum_{k=1}^{\infty} \frac{1}{k} e^{-2 k^2 \pi^2} \sin(2 k \pi x),$$ 
which is equal to
$$h'(x) = \sum_{k=1}^{\infty} 2 \pi e^{-2k^2 \pi^2} \cos(2k \pi x).$$
Note that the function is periodic in $x$ with period $1$, so we can restrict to $x \in [0,1)$. 
Consider $z = 2 \pi x$. 
Suppose first that $|z - \pi/2| > \epsilon$ and $|z-3\pi/2| > \epsilon$ for a
constant $\epsilon > 0$ to be determined. Then, 
$$|\cos(2\pi z)| \geq \cos(\pi/2-\epsilon) = \sin(\epsilon) \geq 2\epsilon/\pi,$$
using that $\cos(\pi/2-\epsilon) = \sin(\epsilon)$ and that 
$\sin(x)/x \geq 2/\pi$ for $0 < x < \pi/2$. In this case, it follows
that 
\begin{eqnarray*}
|h'(x)| & \geq & (2\pi) e^{-2\pi^2} 2\epsilon/\pi - \sum_{k > 1} 2\pi e^{-2k^2 \pi^2}
\geq 4 e^{-2\pi^2} \epsilon - 4\pi e^{-8 \pi^2},
\end{eqnarray*}
using that the summation is dominated by a geometric series. Note that this
expression is at least $4e^{-2\pi^2}(\epsilon - \pi e^{-6\pi^2})$, and so setting
$\epsilon = 2\pi e^{-6\pi^2}$ shows that $|h'(x)| = \Omega(1)$. Notice that $x$
satisfies $|2\pi x - \pi/2| > \epsilon$ provided $|x - 1/4| \geq 1/100 > \epsilon/(2\pi)$ and
that $x$ satisfies $|2\pi x - 3\pi/2| > \epsilon$ provided that $|x-3/4| \geq 1/100 > \epsilon/(2\pi)$. 
As $|\{\bar{\theta}\} - 1/4| \geq 1/50$ and $|\{\bar{\theta}\} - 3/4| \geq 1/50$, it follows that
$x \in [\{\bar{\theta}\} - 1/100, \{\bar{\theta}\} + 1/100]$. 
Hence, $|h'(x)| = \Omega(1)$ for such $x$, as desired. 
%

On the other hand, it is clear that $h'(x)\le 1$, by upper bounding $\cos(2k\pi x)$ by $1$ and using
a geometric series to bound $h'(x)$. 
\end{proof}

\section{Distributed Gap Majority}\label{sec:app:gap_majority}
Our techniques can also be used to obtain a cleaner proof of the lower bound on the information complexity of distributed gap majority due to Woodruff and Zhang \cite{WZ12}. In this problem, there are $k$ parties/machines and the $i^{\text{th}}$ machine receives a bit $z_i$. The machines communicate via a shared blackboard and their goal is to decide whether $\sum_{i=1}^k z_i \le k/2 - \sqrt{k}$ or $\sum_{i=1}^k z_i \ge k/2 + \sqrt{k}$. In \cite{WZ12}, it was proven that the information complexity of this problem is $\Omega(k)$. We give a different proof using strong data processing inequalities. 

The distribution we will consider is the following: let $B \sim B_{1/2}$. Denote $B_{1/2 + 10 /\sqrt{k}}$ by $\mu_1$ and $B_{1/2 - 10 /\sqrt{k}}$ by $\mu_0$. If $B=1$, sample $Z_1,\ldots,Z_k$ according to $\mu_1^k$. If $B=0$, sample $Z_1,\ldots,Z_k$ according to $\mu_0^k$.

\begin{theorem}\label{thm:gap_majority} Suppose $\pi$ is a $k$-party protocol (with inputs $Z_1,\ldots,Z_k$) and $\pi$ solves the gap majority problem (up to some error). Then $I(\Pi; Z_1,\ldots,Z_k|B=0) \ge \Omega(k)$.
\end{theorem}

\noindent $\Pi$ is the random variable for the transcript of the protocol $\pi$. The intuition for the proof is pretty simple. It is not hard to verify that since $\pi$ solves the gap majority problem, it should be able to estimate $B$ as well i.e. $I(\Pi; B) \ge \Omega(1)$. However since each $Z_i$ has only $\Theta(1/k)$ information about $B$, the protocol needs to gather information about $\Omega(k)$ of the $Z_i$'s. It is satisfying that this intuition can indeed be formalized! Perhaps worth noting that similar intuition can be drawn for the two-party gap hamming distance problem but there we don't have a completely information theoretic proof of the linear lower bound \cite{CR11}. We will be using the strong data processing inequality for the binary symmetric channel first proven by \cite{Ahlswede-Gacs}. it studies how information decays on a binary symmetric channel. Suppose $X$ be a bit distributed according to $B_{1/2}$. $Y$ be another bit obtained from $X$ by passing it through a binary symmetric channel with error $1/2-\eps$ (i.e. $Y$ remains $X$ w.p. $1/2 +\eps$ and gets flipped w.p. $1/2 -\eps$). Then for any random variable $U$ s.t. $U - X - Y$ is a Markov chain, $I(U;Y) \le 4 \eps^2 I(U;X)$. 

\begin{proof} We will denote by $\Pi_{b_1,\ldots,b_k}$ the transcript of the protocol $\pi$ when the inputs to $\pi$ are sampled according to $\mu_{b_1} \otimes \mu_{b_2} \otimes \cdots \otimes \mu_{b_k}$. Since $I(\Pi; B) \ge \Omega(1)$, we know that $h^2(\Pi_{0^k}, \Pi_{1^k}) \ge \Omega(1)$. Now
\begin{align*}
I(\Pi; Z_1,\ldots,Z_k| B=0) \ge \sum_{i=1}^k I(\Pi; Z_i | B=0)
\end{align*}
Lets denote our distribution of $\Pi, Z_1,\ldots,Z_k, B$ by $\rho$. We will tweak this distribution a little bit. Take an independent $B' \sim B_{1/2}$. All the variables are distributed the same as $\rho$ except $Z_i$ which is taken to be independently distributed as $\mu_{B'}$. Denote the new distribution as $\rho'$. It is easy to verify that
\begin{align*}
I(\Pi; Z_i|B=0)_{\rho} \ge I(\Pi; Z_i|B=0)_{\rho'}/2
\end{align*}
This is true since in $\rho$, conditioned on $B=0$, $Z_i$ has the distribution $B_{1/2-10/\\sqrt{k}}$ and in $\rho'$ it is $B_{1/2}$ (and hence use Lemma \ref{lem:continuity}). We can also see that
\begin{align*}
I(\Pi; Z_i|B=0)_{\rho'} &\ge \Omega \left( k \cdot I(\Pi; B'|B=0)_{\rho'} \right)\\
&\ge \Omega \left( k \cdot h^2(\Pi_{e_i}, \Pi_{0^k}) \right)
\end{align*}
The first inequality is by strong data processing inequality for the binary symmetric channel and the second by Lemma \ref{lem:info-hellinger}. Now
\begin{align*}
I(\Pi; Z_1,\ldots,Z_k| B=0) &\ge \sum_{i=1}^k I(\Pi; Z_i | B=0) \\
&\ge \sum_{i=1}^k \Omega \left( k \cdot h^2(\Pi_{e_i}, \Pi_{0^k}) \right) \\
&\ge \Omega \left( k \cdot h^2(\Pi_{0^k}, \Pi_{1^k}) \right) \\
&\ge \Omega(k)
\end{align*}
The third inequality is by noting that $\Pi_{b_1,\ldots,b_k}$ satisfies a cut-and-paste property because $\pi$ is a $k$-party protocol and hence Theorem \ref{thm:hellinger-sum} applies. 
\end{proof}
\section{Missing Proofs in Section~\ref{sec:data-processing-ineq}}

\subsection{Proof of Lemma~\ref{lem:checking-log-concave}}

\begin{proof}[Proof of Lemma~\ref{lem:checking-log-concave}]
	Let $u(x)$ be such that $d\mu = \exp(-u(x))dx$, that is,  $u(x) = -\ln(\frac{1}{2}\left(\exp(-u_0(x))+\exp(-u_1(x))\right))$. We calculate $u''(x)$ as follows: 
	
	We can simply calculate the derivatives of $u$. For simplicity of notation, let $h =  \exp(-u_0(x))+\exp(-u_1(x))$. We have that 
	$$h' =  -u_0'\exp(-u_0)-u_1u_1'\exp(-u_1), $$
	and $$h'' = (u_0'^2 - u_0'')\exp(-u_0)+ (u_1'^2 - u_1'')\exp(-u_1). $$
	
	Therefore we have 
	\begin{eqnarray*}
		u'' &=& \frac{-hh'' + h'^2}{h^2} \\
		&=& \frac{u_0''\exp(-2u_0) + u_1''\exp(-2u_1) + (u_0''+u_1'' - (u_0'-u_1')^2)\exp(-u_1-u_2)}{((u_0'^2 - u_0'')\exp(-u_0)+ (u_1'^2 - u_1'')\exp(-u_1))^2} \\
	\end{eqnarray*}
	
	With some simple algebraic manipulations we have that $h'' \ge t$ (for $t\le \min\{\mu_0'',\mu_1''\})$ is equivalent to 
	
	$$\left(\sqrt{\mu_0''-t}\exp(-u_0) - \sqrt{\mu_1''-t}\exp(-u_1)\right)^2 + \left(\left(\sqrt{\mu_0''-t}+\sqrt{\mu_1''-t}\right)^2-(u_0'+u_1')^2\right)\exp(-u_0-u_1)\ge 0$$ 
	
	Therefore, taking $t = \frac{1}{2c}$ and under our assumptions that $|\mu_0'(x)-\mu_1'(x)|\le \sqrt{2c}$ for any $x\in [a,b]$, we have that $u'' \ge \frac{c}{2}$ as desired.  
\end{proof}

\subsection{Proof of Lemma~\ref{suff_statistics}}\label{subsec:suff_statistics}
	Let us look at the density of $(X_1,\ldots, X_n)$ conditioned on $X_1+\dots+X_n = l \le \tau$ and $V = v$. Suppose $x_1,\cdots, x_n$ be such that $\sum_i x_i = l$, then for some normalizing constant $C$
	\begin{align*}
	p(x_1,\cdots,x_n | l, v) &= C \frac{e^{-(x_1 - v\delta)^2/2 \sigma^2} \cdots e^{-(x_n - v\delta)^2/2 \sigma^2}}{e^{-(l-nv\delta)^2/2n\sigma^2}} \\
	&= C e^{(l-nv\delta)^2/2n\sigma^2 - \sum_i (x_i - v\delta)^2/2 \sigma^2} \\
	&= C e^{\frac{(l-nv\delta)^2 - n \sum_i (x_i - v\delta)^2}{2 n \sigma^2}} \\
	&= C e^{\frac{l^2 - n \sum_i x_i^2}{2 n \sigma^2}}
	\end{align*}
	which is independent of $v$ and that proves the lemma. Note that we used the fact that $\sum_i x_i = l$ to simplify the expression. 
	
\subsection{Proof of Lemma~\ref{lem:lipschitz}} \label{subsec:lipschitz}

	The proof is by direct calculation. Note that by definition on support $[-\tau,\tau]$,  $d \mu_0' = \gamma_0\exp(-u_0(x))dx$, and $\d m_0' = \gamma_1\exp(-u_0(x))dx$ with $u_0(x) = -\frac{x^2}{2\sigma^2}$ and $u_1(x) = -\frac{(x-\delta)^2}{2\sigma^2}$, where $\gamma_0$ and $\gamma_1$ are scaling constants. Note that by the definition of the reverse channel $K$, 
	$$f_0(x) = \Pr[V=0\mid X=x] = \frac{\gamma_0e^{-\frac{x^2}{2\sigma^2}}}{\gamma_0 e^{-\frac{x^2}{2\sigma^2}}+\gamma_1e^{-\frac{(x-\delta)^2}{2\sigma^2}}}$$
	
	Therefore 
	$$f_0'(x) = \left(\gamma_0+\gamma_1\exp(\frac{2x\delta-\delta^2}{2\sigma^2})\right)^{-2}\cdot \gamma_0\gamma_1\frac{\delta}{\sigma^2}\exp(\frac{2x\delta-\delta^2}{2\sigma^2})$$
	
	By AM-GM inequality we have 
	$$f_0'(x) \le  \left(4\gamma_0\gamma_1\exp(\frac{2x\delta-\delta^2}{2\sigma^2})\right)^{-1}\cdot \gamma_0\gamma_1\frac{u}{\sigma^2}\exp(\frac{2x\delta-\delta^2}{2\sigma^2}) = \frac{4\delta}{\sigma^2}$$
	
	Similarly for $f_1(v)$ we have 
	$$f_1(x) =  \frac{\gamma_1e^{-\frac{(x-\delta)^2}{2\sigma^2}}}{\gamma_0 e^{-\frac{x^2}{2\sigma^2}}+\gamma_1e^{-\frac{(x-\delta)^2}{2\sigma^2}}}$$
	
	and $$f_1'(x) = \left(\gamma_1+\gamma_0\exp(\frac{-2x\delta+\delta^2}{2\sigma^2})\right)^{-2}\cdot \gamma_0\gamma_1\frac{-\delta}{\sigma^2}\exp(\frac{-2x\delta+\delta^2}{2\sigma^2})\ge\frac{-\delta}{4\sigma^2}$$
	
	Also note that $f_0'\ge 0$ and $f_1' \le 0$. Therefore for any $v$, $f_v'$ is $\frac{\delta}{4\sigma^2}$-Lipschitz

\section{Toolbox}
\begin{lemma}[Folklore, Hellinger v.s. total variation]\label{lem:hellinger-tv}
	For any two distribution $P,Q$, we have 
	$$\h^2(P,Q)\le \TV{P-Q}\le \sqrt{2}\h(P,Q)$$
\end{lemma}
\begin{lemma}\label{lem:info-hellinger}
	Let $\phi(z_1)$  and $\phi(z_2)$ be two random variables. Let Z denote a random variable with uniform distribution in $\{z_1,z_2\}$: Suppose $\phi(z)$ is independent of $Z$ for each $z\in \{z_1,z_2\}$: Then, 
	$$2\h^2(\phi_{z_1},\phi_{z_2})\ge \I(Z;\phi(Z)) \ge \h^2(\phi_{z_1},\phi_{z_2})$$
\end{lemma}
\begin{proof}
	The lower bound of the mutual information follows from Lemma 6.2 of~\cite{DBLP:journals/jcss/Bar-YossefJKS04}. For the upper bound, we assume that for simplicity $\phi$ has discrete support $\mathcal{X}$, though the proof extends continuous random variable directly. We have
	\begin{eqnarray*}
	\I(Z;\phi(Z)) &=& \frac{1}{2}\KL(\phi_1\| (\phi_1+\phi_2)/2) +  \frac{1}{2}\KL(\phi_2\| (\phi_1+\phi_2)/2) \\
	&\le & \frac{1}{2}\chi^2(\phi_1\| (\phi_1+\phi_2)/2) +  \frac{1}{2}\chi^2(\phi_2\| (\phi_1+\phi_2)/2) \\
	&=& \frac{1}{4}\sum_{x\in \mathcal{X}}\frac{(\phi_1(x)-\phi_2(x))^2}{\phi_1(x)+\phi_2(x)} +  \frac{1}{4}\sum_{x\in \mathcal{X}}\frac{(\phi_1(x)-\phi_2(x))^2}{\phi_1(x)+\phi_2(x)}  \\
	&\le& \sum_{x\in \mathcal{X}}\frac{(\phi_1(x)-\phi_2(x))^2}{(\sqrt{\phi_1(x)}+\sqrt{\phi_2(x)})^2} \\
	&=& 2\h^2(\phi_1,\phi_2)
	\end{eqnarray*}
	where the first inequality uses that KL-divergence is less than $\chi^2$ distance and the second one uses the inequality $a^2+b^2 \ge \frac{(a+b)^2}{2}$.
\end{proof}

\begin{theorem}[Corollary of Theorem 7 of \cite{Hellinger}]\label{thm:hellinger-sum}
	Suppose a family of distribution $\{P_{\bs{b}} :\bs{b}\in \{0,1\}^m\}$ satisfies the cut-paste property: for any  for any $\bs{a},\bs{b}$ and $\bs{c},\bs{d}$ with $\{a_i,b_i\} = \{c_i,d_i\}$ (in a multi-set sense) for every $i\in [m]$, $
		\h^2(\Pi_{\bs{a}},\Pi_{\bs{b}}) =\h^2(\Pi_{\bs{c}},\Pi_{\bs{d}})$. 
	Then we have 
	\begin{equation}
	\sum_{i=1}^m \h^2(P_{\bs{0}}, P_{\bs{e_i}}) \ge \Omega(1) \cdot \h^2(P_{\bs{0}}, P_{\bs{1}})
	\end{equation}
	where $\bs{0}$ and $\bs{1}$ are all 0's and all 1's vectors respectively, and $\bs{e_i}$ is the unit vector that only takes 1 in the $i$th entry. 
\end{theorem}

\begin{proof}
	Theorem 7 of~\cite{Hellinger} already proves a stronger version of this theorem for the $m=2^t$ case. Suppose on the other hand $m = 2^t + \ell$ for $\ell <2^t$. We divide $[m] = \{1,\dots,m\}$ into a collection of $2^t$ subsets $A_1,\dots, A_{2^t}$, each of which contains at most 2 elements. Let $\bs{f_i}$ be the indicator vector of the subset $A_i$. For example, if $A_i  = \{p,q\}$, then $\bs{f_i} = \bs{e_p}+\bs{e_q}$. We claim that $\sum_{j\in A_i}\h^2(P_{\bs{0}},P_{\bs{e_j}}) \ge \Omega(1)\h^2(P_{\bs{0}},P_{\bs{f_i}})$. This is trivial when $|A_i| = 1$ and when $A_i = \{p,q\}$, we have that by Cauchy–Schwarz inequality and the cut-paste property
	$$\h^2(P_{\bs{0}},P_{\bs{e_p}}) + \h^2(P_{\bs{0}},P_{\bs{e_q}}) \ge \frac{1}{2}\h^2(P_{\bs{e_p}},P_{\bs{e_q}})  = \frac{1}{2}\h^2(P_{\bs{0}},P_{\bs{e_q+e_q}}).$$   
	Therefore, we can lowerbound LHS as 
	$$\sum_{i=1}^m \h^2(P_{\bs{0}}, P_{\bs{e_i}}) \ge \frac{1}{2}\sum_{i=1}^{2^t}\h^2(P_{\bs{0}},P_{\bs{f_i}}).$$
	Then applying Theorem 7 of~\cite{Hellinger} on the RHS of the inequality above we have 
	$$\frac{1}{2}\sum_{i=1}^{2^t}\h^2(P_{\bs{0}},P_{\bs{f_i}}) \ge \Omega(1)\cdot\h^2(P_{\bs{0}}, P_{\bs{1}}),$$
	and the theorem follows. 
\end{proof}

\begin{lemma}\label{lem:continuity}
	Suppose two distribution $\calD,\calD'$ satisfies $\calD \ge c\cdot\calD'$.  Let $\Pi(X)$ be a random function that only depends on $X$. If $X\sim \calD$ and $X'\sim \calD'$, then we have that 
	\begin{equation}
	\I(X;\Pi(X))\ge c\cdot\I(X';\Pi(X'))
	\end{equation}
\end{lemma}
\begin{proof}
	Since $\calD \ge c\cdot\calD'$, we have that 
	$$\I(X;\Pi(X)) = \Exp_{X\sim \calD}\left[\KL(\Pi_X \| \Pi) \right] \ge  c\cdot\Exp_{X'\sim \calD'}\left[\KL(\Pi_{X'} \| \Pi)\right] $$
	Then note that 
	$$\Exp_{X'\sim \calD'}\left[\KL(\Pi_{X'} \| \Pi)\right]  = \Exp_{X'\sim \calD'}\left[\KL(\Pi_{X'} \| \Pi')\right]  + \KL(\Pi'\|\Pi) $$
	
	It follows that 
		$$\I(X;\Pi(X)) \ge c\cdot\Exp_{X\sim \calD'}\left[\KL(\Pi_X \| \Pi')\right]  = c\cdot\I(X';\Pi(X'))$$
\end{proof}
\begin{lemma}[Folklore]\label{lem:mutual-info-additivity}
	When $X$ is drawn from a product distribution, then 
	$$\sum_{i=1}^m \I(X_i;\Pi) \le \I(X;\Pi)$$
\end{lemma}

%

\end{document}